\documentclass{article}
\usepackage{iclr2027_conference,times}

\usepackage{amsmath,amssymb,mathtools,amsthm,mathrsfs}
\usepackage{thmtools}
\usepackage{thm-restate}
\usepackage{algorithm}
\usepackage[noend]{algpseudocode}
\usepackage{xcolor}
\algrenewcommand{\algorithmiccomment}[1]{%
  \hfill\textcolor{black!55}{$\triangleright$~#1}%
}
\usepackage{booktabs}
\usepackage{graphicx}
\usepackage{placeins}
\usepackage{microtype}
\DeclareMicrotypeAlias{Noto Color Emoji}{TU-empty}
\usepackage{hyperref}
\usepackage{url}
\hypersetup{colorlinks=true,allcolors=blue}
\usepackage[nameinlink,noabbrev]{cleveref}
\usepackage{verbatim}
\creflabelformat{equation}{#2{}#1{}#3}
\crefname{equation}{eq.}{eqs.}
\Crefname{equation}{Eq.}{Eqs.}

\usepackage{etoc}

\usepackage{amsmath,amsfonts,bm}

\def\mH{{\bm{H}}}

\def\mX{{\bm{X}}}

\renewcommand{\epsilon}{\varepsilon}

\newcommand{\D}{\mathscr{D}}

\renewcommand{\O}{\mathcal{O}}

\newcommand{\reals}{\mathbb{R}}

\usepackage{amsmath}
\usepackage{xcolor}
\usepackage{tcolorbox}
\tcbuselibrary{theorems,skins}

\definecolor{equationblue}{HTML}{365F82}

\tcbset{
  equationbox/.style={
    enhanced,
    colback=equationblue!5!white,
    colframe=equationblue!45!white,
    boxrule=0pt,
    arc=2mm,
    boxsep=0pt,
    left=10pt,
    right=10pt,
    top=7pt,
    bottom=7pt,
    colbacktitle=equationblue!5!white,
    coltitle=equationblue,
    fonttitle=\small\sffamily\bfseries,
    titlerule=0pt,
    toptitle=6pt,
    bottomtitle=2pt,
  }
}

\renewcommand{\boxed}[2][]{%
  \tcboxmath[equationbox,title={#1}]{#2}%
}

\usepackage{thmtools}

\newcommand{\red}[1]{\textcolor[rgb]{0.635,0.0780,0.1840}{#1}}

\definecolor{shade}{rgb}{0.9,0.9,0.9}

\DeclareMathOperator*{\argmin}{arg\,min}

\newcommand{\emoji}[1]{%
  \raisebox{-0.1em}{%
    \includegraphics[height=1em]{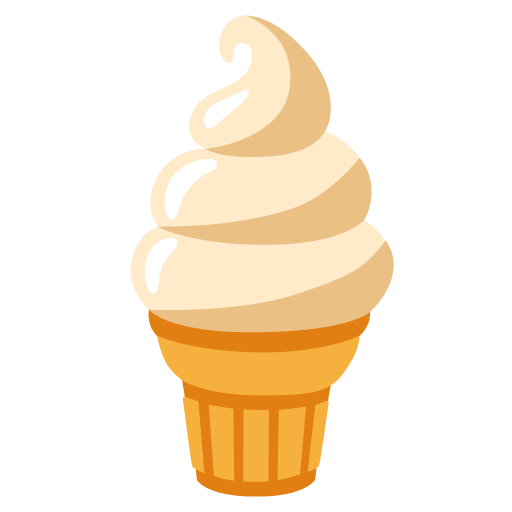}%
  }%
}

\newcommand{\mat}[1]{\bm{#1}}

\newcommand{\R}{\mathbb{R}}
\newcommand{\tr}{\operatorname{tr}}
\newcommand{\diag}{\operatorname{diag}}
\newcommand{\vecop}{\operatorname{vec}}

\newcommand{\SoftServe}{\textnormal{\textsc{SoftServe}}}
\newcommand{\QME}{\operatorname{QME}}
\newcommand{\norm}[2][]{\left\lVert #2\right\rVert_{#1}}

\newcommand{\Splusplus}{\mathbb{S}_{\scalebox{0.7}{$\scriptscriptstyle ++$}}}

\usepackage[colorinlistoftodos]{todonotes}

\newcommand{\rob}[1]{\todo[inline,bordercolor=blue,backgroundcolor=blue!20,linecolor=blue]{\textbf{Robert: }#1}}

\usepackage[framemethod=tikz]{mdframed}

\definecolor{shadecolor}{RGB}{244,248,253}
\definecolor{shadeborder}{RGB}{190,207,225}

\declaretheoremstyle[
  headfont=\normalfont\bfseries,
  notefont=\mdseries,
  notebraces={(}{)},
  bodyfont=\normalfont,
  postheadspace=0.5em,
  spaceabove=6pt,
  mdframed={
    skipabove=8pt,
    skipbelow=8pt,
    hidealllines=false,
    backgroundcolor=shadecolor,
    linecolor=shadeborder,
    linewidth=0.5pt,
    roundcorner=4pt,
    innerleftmargin=10pt,
    innerrightmargin=10pt,
    innertopmargin=7pt,
    innerbottommargin=7pt
  }
]{shaded}

\declaretheorem[style=shaded,numberwithin=section]{definition}

\declaretheorem[style=shaded,sibling=definition]{lemma}

\title{%
  \SoftServe{}{\normalfont\emoji{soft-ice-cream}}:
  A Scalable \\
  Quasi-Newton Method for Deep Learning%
}
\author{
Joohwan Ko\textsuperscript{1,$\dagger$}
\quad
Tetiana Parshakova\textsuperscript{2}
\quad
Diana Cai\textsuperscript{3,$\ddagger$}
\quad
Robert M.~Gower\textsuperscript{2,$\ddagger$}
\\[0.5em]
{\normalfont
\textsuperscript{1}University of Massachusetts Amherst
\quad
\textsuperscript{2}CCM, Flatiron Institute
\quad
\textsuperscript{3}Cornell University
}
}

\iclrfinalcopy
\begin{document}
\etocdepthtag.toc{main}
\maketitle

{\renewcommand{\thefootnote}{\fnsymbol{footnote}}
\footnotetext[2]{Work done while at the Flatiron Institute. Email: \texttt{joohwanko@cs.umass.edu}.}
\footnotetext[3]{Equal senior authors.}
}
\begin{abstract}

Quasi-Newton (QN) methods have long been among the most effective methods for large-scale unconstrained convex optimization. Two obstacles have limited their use in deep learning: non-convexity and  enormous parameter sizes. 
We introduce \SoftServe{}, a family of QN methods designed to overcome these obstacles without line searches or ad hoc curvature corrections.
\SoftServe{} derives positive-definite curvature estimates from the variational objective of~\citet{berglund2024}, even in the presence of negative curvature.
We develop diagonal and Kronecker-factored variants that preserve positive definiteness by construction and scale to 
massive neural networks.
Finally, \SoftServe{} relies on the stable coupled Newton--Schulz iteration for the required matrix operations, replacing costly matrix decompositions with GPU-friendly matrix multiplications.
\SoftServe{} excels on problems that are severely ill-conditioned, including tasks such as recurrent networks, deep autoencoders, physics-informed
neural networks, and a 136M-parameter physics-informed diffusion model, often achieving lower losses than established baselines including Adam, Muon, and SOAP.
\end{abstract}

\section{Introduction}
\label{sec:introduction}

Quasi-Newton methods, such as the celebrated BFGS method \citep{broyden1965convergence,fletcher1970new,goldfarb1970family,shanno1971conditioning}, are among the most effective methods for 
large-scale, 
convex deterministic optimization.
Deep learning  has led to renewed interest
in adapting these methods to train large networks \citep{goldfarb2020practical,ren2021kronecker}.
However, applying these methods to deep learning
requires  curvature estimates
that can accommodate non-convex objectives and gradient noise while remaining computationally efficient.

To see how curvature enters the optimization update,
consider a generic second-order method.
For parameters  \(x \in \R^D\),   second-order methods typically take steps of the form
\begin{equation*}
 x^+ = x -\gamma \mat{H}g.
\end{equation*}
Here \(x^+\) is the updated parameter vector, and \(\gamma>0\) is the step size.
The vector \(g \in \R^D\) is the (stochastic) gradient at \(x\), and the matrix
\(\mat{H} \in \R^{D\times D}\) is an estimate of the inverse Hessian.

Non-convexity of the objective 
poses a fundamental challenge for second-order optimization methods because the true Hessian could have negative eigenvalues. 
If the estimate~$\mat{H}$ has a negative eigenvalue, then the update direction~$-\mat{H}g$ could be an ascent direction, which would increase the loss function. 
This is one of the reasons why first-order methods are often preferred in deep learning: one can always guarantee descent (in expectation). 
For a second-order method to be applicable in deep learning, the estimate $\mat{H}$ must always be positive definite, even in the presence of directions of negative curvature.

For quasi-Newton methods, such as BFGS, maintaining positive definiteness is at odds with the main tool used to measure curvature, namely the \emph{secant equation}. 
In the presence of negative curvature, enforcing the secant equation can inadvertently force the Hessian estimate to become indefinite, as we will show in~\Cref{sec:preliminaries}.
The standard remedy 
is to use sophisticated line search procedures to only move in directions of positive curvature \citep{nocedal2006numerical,Wolfe1969}.
But these line searches require access to the full batch loss, which is too costly and not feasible in the stochastic setting of deep learning.

Beyond the challenges arising from non-convexity
and stochastic objectives,
the sheer scale of modern deep learning models makes storing and updating dense curvature estimates prohibitively expensive.
A standard approach for reducing this cost is the limited memory approach of L-BFGS \citep{liu1989limited},  
which typically uses only 10 pairs of parameter and gradient differences to implicitly build the estimate~$\mat{H}$. 
But storing 10 pairs of parameters and gradients would require storing 10 copies of our model and 10 copies of the gradient, increasing the memory storage by 20-fold. This is infeasible for all but small models.

To scale to modern deep learning, 
it is common to approximate the matrix~$\mat{H}$ by a diagonal matrix~\citep{duchi_adaptive_2011,tieleman2012rmsprop,kingma_adam_2015} or a Kronecker-structured matrix~\citep{kfac,gupta2018shampoo}.
For quasi-Newton optimization,
while
Kronecker-structured updates have been proposed for applying BFGS to deep learning \citep{goldfarb2020practical,ren2021kronecker},
this approach still needs to be 
applied with mechanisms for handling non-convexity
and stochasticity, limiting its applicability.

A promising new direction 
is to reconsider how curvature enters the quasi-Newton update.
Recently, \citet{berglund2024} show that by replacing the
secant equation with a soft secant penalty we can ensure that our Hessian estimate is positive definite, even in the presence of negative curvature.
But \citet{berglund2024} only consider dense estimates,
making their approach infeasible for deep learning.

We introduce \SoftServe{} (Soft \textbf{S}ecant \textbf{E}quation \textbf{R}egularized with \textbf{V}ariational \textbf{E}ntropy), a quasi-Newton framework that combines the flexibility of a soft secant constraint~\citep{PSGD-Lie,berglund2024} with the computational advantages of a structured matrix approximation of the inverse Hessian.
Optimizing within diagonal and Kronecker-structured families yields novel 
updates that 
preserve positive definiteness, without full batch evaluations or line searches.
We develop efficient algorithms for computing these updates using modern GPU-friendly numerical linear algebra methods, such as variants of Newton-Schulz~\citep{polarexpress,higham_functions_2008}, allowing \SoftServe{} to scale to large neural networks.  

We then demonstrate the utility of \SoftServe{} on several deep learning tasks where curvature information is fundamental---such as severely ill-conditioned tasks---including recurrent networks, autoencoders, and physics-informed models (such as a 136M-parameter diffusion model).
On these tasks we find that \SoftServe{} often outperforms modern methods such as Adam, Muon, and SOAP. 
A Python implementation of 
\SoftServe{} is available at \url{https://github.com/joohwanko/SoftServe}.

\section{Quasi-Newton Preliminaries}
\label{sec:preliminaries}

Here we introduce our notation and review important quasi-Newton~(QN)
ideas that are relevant to \SoftServe{}.
For simplicity, in this section we represent all parameters as a flattened vector
 \(x\in\R^D\). We introduce the setting of deep learning in subsequent sections, where parameters can be matrices. 

\vspace{-10pt} \paragraph{Notation.}
We write \(\mathbb S_{++}^d\) and \(\mathbb S_+^d\) for the symmetric
positive definite and positive semidefinite \(d\times d\) matrices, \(\mat{I}_d\) for
the identity, and we define \(\norm[\mat{M}]{v}^2 :=v^\top \mat{M}v\) for \(\mat{M}\in\mathbb S_{++}^d\).

\subsection{The BFGS algorithm}
\label{sec:bfgs}

Let $f$ be a continuous,
twice differentiable function on $\mathbb{R}^D$.
Given  iterates~$x_t$ and gradient evaluations~$\nabla f(x_t)$,
quasi-Newton methods use the following iterate and gradient differences
\begin{align*}
 s_t &:= x_{t+1}-x_t,
 \\
 y_t &:= \nabla f(x_{t+1})-\nabla f(x_t),
 \label{eq:generic-secant}
\end{align*}
to estimate curvature. Indeed, using Taylor's theorem we have that
\begin{equation}
 y_t  = \mat{B}_t s_t,
\label{eq:taylor}
\end{equation}
where $\mat{B}_t:=\int_0^1 \nabla^2 f(x_t+\tau s_t)\,d\tau$ is the averaged Hessian.
Quasi-Newton methods use~\Cref{eq:taylor} to find an estimate of the inverse Hessian
$\mat{H}$ by enforcing the \emph{secant equation} 
\begin{align} \label{eq:secant}
   \boxed{ \mat{H} y_t = s_t.}
\end{align}
However, the secant equation is underdetermined:\ 
$\mat{H}$ has $D^2$ unknown entries, while~\Cref{eq:secant} imposes only $D$ constraints. 
Thus, to obtain a unique update, one must impose an additional regularization criterion.
Quasi-Newton methods resolve this by choosing $\mH$ that both satisfies the secant equation and is close to the previous estimate $\mH_t$ in some sense.

For instance, the BFGS method chooses the matrix~$\mat{H}$
that satisfies the secant equation and is closest to~$\mat{H}_t$
in a weighted Frobenius norm \citep{nocedal2006numerical}; see
Appendix~\ref{app:bfgs}.
We adopt
an equivalent characterization of the BFGS update due to \citet{FletcherKL},
where closeness is measured by the
\emph{log-determinant divergence}.
This formulation will be the starting point for
\SoftServe{}.
For \(\mat{P},\mat{Q}\in\mathbb S_{++}^d\) 
the log-determinant divergence is defined as
\begin{equation}
 \D(\mat{P}; \mat{Q}) :=
 \tr(\mat{P}\mat{Q}^{-1})-\log\det(\mat{P}\mat{Q}^{-1})-d.
 \label{eq:ld-divergence}
\end{equation}
This divergence is strictly convex in $\mat{P}$ (given fixed $\mat{Q}$)
and also is equivalent (up to a constant) to the \emph{relative entropy}, or
Kullback-Leibler divergence between two Gaussians with the same mean and covariances~$\mat{P}$ and $\mat{Q}$.
\citet{FletcherKL} showed that, when \(s_t^\top y_t>0\),
the BFGS update is given by the solution of the following optimization problem:
\begin{align}
 \label{eq:bfgs-inverse-variational}
\boxed{
 \mat{H}_{t+1}^{\rm BFGS}=\underset{\mat{H}\in \mathbb S_{\scalebox{0.7}{$\scriptscriptstyle ++$}}^D}{\arg\min}~\D(\mat{H}_t; \mat{H})
 \quad \text{subject to}\quad \mat{H}y_t=s_t.
}
\end{align}

\paragraph{What happens in the non-convex and stochastic setting?}
The requirement  \(s_t^\top y_t>0\) is called the \emph{curvature condition} and it
always holds for convex objectives, but can fail for nonconvex objectives. 
Indeed, if $s_t^\top y_t <0$,
then any matrix~$\mat{H}$ satisfying the secant equation must also satisfy $s_t^\top y_t = y_t^\top \mat{H} y_t <0,$
and thus cannot be positive definite.

In stochastic settings, the construction of the gradient difference $y_t$ requires additional care.
Let $\nabla f_\xi(x)$ denote a stochastic gradient evaluated
with random draw $\xi$, which specifies the minibatch and other sources of randomness such as data augmentation or dropout masks.
In this case, we compute
$y_t = \nabla f_\xi(x_{t+1})-\nabla f_\xi(x_t)$
using the same draw at both endpoints,
in order to ensure that the gradient difference
reflects the curvature for a single sampled loss. Enforcing the secant equation~\Cref{eq:secant} in the stochastic setting is also problematic, since~\Cref{eq:secant} would force $\mH$ to exactly fit the curvature of just one stochastic loss.

\subsection{Soft quasi-Newton}
\label{sec:soft-qn}


Soft QN (quasi-Newton)
builds on the alternative variational 
formulation of BFGS given in
\Cref{eq:bfgs-inverse-variational},
relaxing the secant constraint by transforming it into a least-squares penalty in the objective function~\citep{berglund2024}. Given $\lambda > 0$ and the current estimate $\mat{H}_t$,
the Soft QN algorithm  approximates the inverse Hessian by solving
\begin{align}
\boxed{
 \mat{H}_{t+1}=\underset{\mat{H}\in\Splusplus^D}{\arg\min}~
       \norm[\mat{H}^{-1}]{s_t-\mat{H}y_t}^2
 + \tfrac{1}{\lambda}\D(\mat{H}_t; \mat{H}).
 }
\label{eq:soft-qn-objective}
\end{align}
The first term in \Cref{eq:soft-qn-objective} penalizes the secant residual; the second term measures how far $\mat{H}_t$ is from $\mat{H}$ in log-determinant divergence~\Cref{eq:ld-divergence}. The parameter \(\lambda > 0\) controls the tradeoff between the two terms.

\Cref{eq:soft-qn-objective} has several nice properties.
First, it is scale invariant
and strictly convex in $\mat{H}^{-1}$.
Second,
there is a closed-form solution 
for $\mat{H}_{t+1}$, given in \Cref{eq:dense-soft-qn-closed-form}.
Third, the BFGS update
is also a solution to~\Cref{eq:soft-qn-objective} in the limit as \(\lambda\to\infty\) and \(s_t^\top y_t>0\) (see~\citet{berglund2024} and Appendix~\ref{app:soft-quasi-newton}).
Finally 
and crucially,
Soft QN 
remains well-defined and positive definite for every finite pair, including those with negative curvature \(s_t^\top y_t\leq0\).  

However, 
computing the solution of \Cref{eq:soft-qn-objective} costs \(\O(D^2)\)  operations, which is computationally expensive, and the memory cost of storing the $D^2$ matrix $\mat{H}$ is prohibitive in deep learning.
We circumvent both the operational and memory costs by restricting  to structured matrices.

\section{\SoftServe{}:\ Quasi-Newton Updates over Structured Families}
\label{sec:method}

To scale to deep neural networks, we consider two choices 
for structured families \(\mathcal H\subset\mathbb S_{++}^D \):
 a diagonal matrix family
and Kronecker-structured matrix family.
\SoftServe{} then solves \Cref{eq:soft-qn-objective} for an inverse Hessian approximation $\mat{H}$ over the structured family $\mathcal{H}$,
instead of over dense matrices in $\mathbb S_{++}^D$.
This optimization 
yields novel and efficiently-computable
updates, which we call \SoftServe{}-Diag
and \SoftServe{}-Kron. 

Instead of the $\O(D^2)$ cost of Soft QN,
\SoftServe{}-Diag 
reduces the memory and computational cost to \(\O(D)\), which is comparable to that of Adam~\citep{kingma_adam_2015}. 
\SoftServe{}-Kron
reduces the costs to \(\O(n^2+m^2)\) for each \(D=n\times m\) parameter block, which is the cost of Shampoo~\citep{gupta2018shampoo} and KFAC~\citep{kfac}.

\subsection{Diagonal approximation}
\label{sec:diag}
\SoftServe{}-Diag restricts the inverse Hessian estimate~$\mat{H}$ to be diagonal. 
Thus we need only store
one positive scale per coordinate, \(\mat{H}=\diag(h)\) with \(h\in\R_{++}^D\). Both terms of the objective in~\Cref{eq:soft-qn-objective} then split into sums over coordinates $h_{t+1,i}$.
Each coordinate is a scalar problem with a unique, closed-form solution $h_{t+1,i} > 0$ (\Cref{lem:diag-update}), given by
\begin{align} 
\label{eq:diag-update1}
\boxed{
 h_{t+1,i}=
 \frac{2\left(h_{t,i}+\lambda s_{t,i}^2\right)}
 {1+\sqrt{1+4\lambda y_{t,i}^2
                  \left(h_{t,i}+\lambda s_{t,i}^2\right)}},
 \qquad i=1,\ldots,D.
 }
\end{align}

Thus, the matrix estimate is $\mat{H}_{t+1}=\diag(h_{t+1})$.
Though \SoftServe{}-Diag is inexpensive, and 
can perform reasonably well, 
it ignores all interactions between coordinates.

\subsection{Kronecker-structured approximation}
\label{sec:kron}

Neural network parameters are often naturally represented
as matrices.
Let $\mat{X} \in \R^{n\times m}$ denote a layer's weight matrix and $\mat{G}:=\nabla_{\mat{X}}f \in \R^{n\times m}$
denote its gradient.
Thus the Hessian matrix associated with $\mat{X}$ 
has $n^2m^2$ entries. 
We can reduce this storage cost by imposing a Kronecker structure: 
\begin{equation}
 \boxed{\mat{H}=\mat{L}\otimes \mat{R},\quad \mbox{where} \quad \mat{L}\in\mathbb S_{++}^m,\quad \mat{R}\in\mathbb S_{++}^n.}
 \label{eq:kron-metric}
\end{equation}
Here $\otimes$ denotes the \emph{Kronecker product}:
if $\mat{L} = [a_{ij}]$, then $\mat{L} \otimes \mat{R}$ is the
$nm\times nm$ block matrix whose $(i,j)^{\text{th}}$ 
block is $a_{ij} \mat{R}$.
Storing the two factors $\mat{L}$ and $\mat{R}$ requires only
$m^2 + n^2$ entries.
By \Cref{eq:kron-vec}, applying \(\mat{H}\) to the gradient also requires only two 
matrix products:
$\mat{H}g=\vecop(\mat{R}\mat{G}\mat{L}),$
where \(g=\vecop(\mat{G})\) and $\vecop(\cdot)$ stacks columns.

Unfortunately  there is no closed-form solution when minimizing the soft quasi-Newton objective~\Cref{eq:soft-qn-objective} over both $\mat{L}$ and $\mat{R}$.  But there is a unique solution if we solve for only one factor at a time.
We therefore update them alternately,
holding one fixed while optimizing the other.
As we now show, each conditional minimization reduces to solving a quadratic matrix equation with a unique positive-definite solution (see \Cref{lem:kron-sweep}).

To see this, we first reshape the secant vectors
into matrices $\mat{S}_t, \mat{Y}_t \in \reals^{n \times m}$,
with
    $\vecop(\mat{S}_t) = s_t$ and
    $\vecop(\mat{Y}_t) =  y_t.$
Holding $\mat{R}=\mat{R}_t$ fixed and 
minimizing the  terms in~\Cref{eq:soft-qn-objective} involving $\mat{L}$,
we obtain the conditional update
\begin{align}
\boxed{
\mat{L}_{t+1} =
    \argmin_{\mat{L} \in \Splusplus^m} \quad\tr(\mat{U}_L \mat{L}) + \tr(\mat{V}_L \mat{L}^{-1}) + \log\det \mat{L},
    }
\label{eq:factorLupdate-main}
\end{align}
where we define the
auxiliary matrices:
\begin{equation} \label{eq:ULVL}
    \mat{U}_L := \tfrac{\lambda}{n} \mat{Y}_t^\top \mat{R}_t \mat{Y}_t \quad \mbox{and} \quad \mat{V}_L := \mat{L}_t + \tfrac{\lambda}{n} \mat{S}_t^\top \mat{R}_t^{-1} \mat{S}_t.
\end{equation}
Differentiating~\Cref{eq:factorLupdate-main} with respect to $\mat{L}$ and setting the derivative to zero, yields the  quadratic matrix equation
\begin{align}
    \boxed{
    \mat{L} \, \mat{U}_L \, \mat{L} + \mat{L} = \mat{V}_L.
    }
        \label{eq:factorA-main}
\end{align}
Here, because $\mat{U}_L \succeq 0$ and $\mat{V}_L \succ 0$, this equation has a unique, positive-definite solution~$\mat{L}_{t+1}$, which we provide in \Cref{eq:qme-symmetric-main} and discuss  further in \Cref{sec:gram}.

We then update the factor $\mat{R}$ by minimizing~\Cref{eq:soft-qn-objective} in $\mat{R}$ and holding $\mat{L}_{t+1}$ fixed:
\begin{align}
\boxed{
\mat{R}_{t+1} =
    \argmin_{\mat{R} \in \Splusplus^n} \quad\tr(\mat{U}_R \mat{R}) + \tr(\mat{V}_R \mat{R}^{-1}) + \log\det\mat{R},
    }
    \label{eq:factorRupdate-main}
\end{align}
where the auxiliary matrices
are now defined as
\begin{equation} \label{eq:URVR}
\mat{U}_R := \tfrac{\lambda}{m} \mat{Y}_t \mat{L}_{t+1} \mat{Y}_t^\top \quad \mbox{and} 
\quad \mat{V}_R := m^{-1} [\tr(\mat{L}_t \mat{L}_{t+1}^{-1}) \mat{R}_t + \lambda \mat{S}_t \mat{L}_{t+1}^{-1} \mat{S}_t^\top].\end{equation}

Solving~\Cref{eq:factorRupdate-main}
    for $\mat{R}_{t+1}$
    reduces to solving an analogous quadratic matrix equation to~\Cref{eq:factorA-main}, $\mat{R} \, \mat{U}_R \, \mat{R} + \mat{R} = \mat{V}_R$, which again has a unique, positive-definite solution (\Cref{lem:kron-sweep}).

The alternating updates replace computations on the full inverse-Hessian estimate $\mat{H}$ with operations on the two smaller factors $\mat{L}$ and $\mat{R}$. 
Nevertheless, directly solving the quadratic matrix equations can be
expensive for wide layers, with  cost
$\O(n^3 + m^3 + n^2m + nm^2)$. 
Instead of using a direct solver, we  will use GPU-friendly Newton-Schulz approximations, as discussed in the next section. 
Using Newton-Schulz approximations was also central to Muon's effectiveness at scale~\citep{polarexpress,liu2025muon}.

\section{Practical implementation of \SoftServe{} for deep learning}

Here we discuss strategies
to make \SoftServe{}
practical for training deep neural networks.

\subsection{Efficient Kronecker updates}
\label{sec:gram}

\SoftServe{}-Kron
requires solving quadratic matrix equations \citep{cai2024},
which are particular instances of the
algebraic Riccati equation~\citep{lancaster1995algebraic}.
Here we discuss how to approximately solve this matrix equation using a GPU friendly approach. These approximations are key for scaling  \SoftServe{} up to large neural networks. 

Both Kronecker factor updates in \Cref{eq:factorLupdate-main} and \Cref{eq:factorRupdate-main}
reduce to solving a quadratic matrix equation of the form $\mat{Q}\mat{U}\mat{Q}+\mat{Q}=\mat{V}$
with $\mat{U} \succeq 0$ and $\mat{V} \succ 0$.
Its unique positive-definite solution can be expressed using principal matrix square roots~$\mat{V}^{1/2}$ (see Appendix~\ref{app:qme}):
\begin{equation}
\boxed{
\mat{Q}
 =2\mat{V}^{1/2}
\left[\mat{I}_d+\left(\mat{I}_d+4\mat{V}^{1/2}\mat{U}\mat{V}^{1/2}\right)^{1/2}\right]^{-1}
 \mat{V}^{1/2}.
 }
 \label{eq:qme-symmetric-main}
\end{equation}

Direct evaluation of  \Cref{eq:qme-symmetric-main}
requires two matrix square roots and a matrix inverse.
For Kronecker updates,
the coefficients admit a natural Gram matrix factorization
$\mat{U} = \mat{F}\mat{F}^\top$,
given in \Cref{app:gram:details}.
Our implementation exploits this structure through
an equivalent 
Gram formulation of \Cref{eq:qme-symmetric-main} given in \Cref{eq:gram:app},
which evaluates a square root of $\mat{I} + 4\mat{F}^\top \mat{V} \mat{F}$ and an inverse,
which avoids the separate square roots of $\mat{V}$.
See Appendix~\ref{app:qme} and \ref{app:gram:details} for details.

While the matrix functions in~\Cref{eq:qme-symmetric-main} can be computed
using, for instance singular value decomposition (SVD), 
most methods for computing the SVD are not well suited for GPUs~\citep{svdnoGPU}. 
To compute the matrix square root efficiently on 
GPUs, we use the coupled Newton-Schulz iteration 
\citep{higham_stable_1997,higham_functions_2008},
which is based only on
matrix multiplications and additions.
For a positive definite matrix $\mat{M}$, 
coupled Newton-Schulz jointly approximates its square root~$\mat{M}^{\frac{1}{2}}$
and its inverse square root~$\mat{M}^{-\frac{1}{2}}$.
Matrix inverses are likewise approximated 
using Newton-Schulz iterations. 
See Appendix~\ref{asec:newton-schulz-couple} for details.

\subsection{Training with \SoftServe{}}

\label{sec:stochastic}

To make \SoftServe{} practical, we  make use of several now-standard practices in deep learning.
The full procedure includes momentum updates,
Nesterov extrapolation or bias correction \citep{kingma_adam_2015},  details of which are in Appendix~\ref{app:training:details}.

\begin{algorithm}[t]
\small
\caption{\SoftServe{}-Kron:\ high-level training iteration (details in \Cref{alg:softserve-kron})}
\label{alg:softserve1}
\begin{algorithmic}[1]
\Require Parameters $\mat{X}_t$; factors $\mat{L}_t,\mat{R}_t$;
momentum estimate $\hat{\mat{M}}_{t}$; step size $\gamma_t$;
penalty $\lambda$.

\State Take the normalized preconditioned step in
\Cref{eq:normalization}, using
\[\mat{X}_{t+1} = \mat{X}_t -\gamma_t\frac{\mat{R}_t\hat{\mat{M}}_t\mat{L}_t}{\sqrt{\tr\left( \hat{\mat{M}}_t^\top\mat{R}_t\hat{\mat{M}}_t\mat{L}_t\right) }} \] \vspace{-0.3cm}
\label{ln:normalized-kron-update}
 
\State Evaluate stochastic gradient, difference of gradients $\mat{Y}_t$ and difference of iterates $\mat{S}_t$
\State Update Kronecker factors by solving quadratic matrix equations (via \Cref{eq:qme-symmetric-main} or \Cref{eq:gram:app}):
\begin{align*}
    &\text{Update~$\mat{L}_{t+1}$~by~solving~}
    \mat{L} \, \mat{U}_L \, \mat{L} + \mat{L} = \mat{V}_L
    \\
    &\text{Update~$\mat{R}_{t+1}$~by~solving~}
    \mat{R} \, \mat{U}_R \, \mat{R} + \mat{R} = \mat{V}_R,
\end{align*}
where the matrices $\mat{U}_L, \mat{V}_L$ are given in~\Cref{eq:ULVL}, and  $\mat{U}_R, \mat{V}_R$ are given in~\Cref{eq:URVR}.
\State \Return $\mat{X}_{t+1}$
\end{algorithmic}
\end{algorithm}

\Cref{alg:softserve1} summarizes one training iteration of \SoftServe{}-Kron applied to one matrix parameter $\mX \in \R^{n \times m}.$
\SoftServe{}-Diag follows the same procedure 
replacing the Kronecker updates with the diagonal update in \Cref{eq:diag-update1}, and detailed in~\Cref{alg:softserve-diag}

One of main changes we have introduced in~\Cref{alg:softserve1} is that we use a \emph{normalized} update, see line~\ref{ln:normalized-kron-update}. In vectorized notation, this normalized update is given by
\begin{align}
    x_{t+1} = x_t - \gamma_t \mat{H}_t \hat{m}_t \,\Big/\sqrt{\hat{m}_t^\top \mat{H}_t \hat{m}_t},
    \label{eq:normalization}
\end{align}
where $x_t = \vecop(\mat{X}_t), \hat{m}_t = \vecop(\hat{\mat{M}}_t) $ and $\mat{H}_t = \mat{L}_t\otimes \mat{R}_t.$ By \Cref{eq:kron-vec,eq:vec-inner-product}, \Cref{eq:normalization} is equivalent to the matrix update on line~\ref{ln:normalized-kron-update} of \Cref{alg:softserve1}.
Normalized updates are often more effective in deep learning~\citep{jordan2024muon,scion,crawshaw2025muonvariants}, which we also confirmed in our initial experiments.

Finally, to amortize the cost of refreshing the curvature
estimate, we hold $\mat{H}_t$ fixed for $K$ training steps. 
At the beginning of each window, we save the parameters,
stochastic gradient, and random draw.
At the end, we reevaluate the gradient using the saved draw,
forming a secant pair from the endpoint differences, see Appendix~\ref{app:training:details} for details.

\section{Related Work}
\label{sec:related}

\paragraph{Stochastic and structured Quasi-Newton.}
Classical QN methods enforce the secant equation while changing the metric as little as possible
\citep{broyden1965convergence,fletcher1970new,goldfarb1970family,shanno1971conditioning,dennis_jr_quasi-newton_1977,nocedal2006numerical}. Stochastic variants keep a dense or limited-memory metric \citep{liu1989limited} and protect it from bad pairs by: using damping \citep{schraudolph_stochastic_2007}, curvature pairs from Hessian-vector products on large batches \citep{byrd_stochastic_2016}, damped updates ~\citep{wang_stochastic_2017}, overlapping batches \citep{berahas_multi-batch_2016},  diagonal rescaling \citep{bordes2009sgd}, or a secant penalty \citep{irwin_secant_2020}. 
K-BFGS \citep{goldfarb2020practical} uses a per-layer Kronecker factorization of BFGS, and relies on damping for a resulting positive definite curvature estimate.

\paragraph{Matrix-aware methods.}
There are now a growing number of methods which are specialized to the matrix parameters in linear layers. 
K-FAC \citep{kfac}, Shampoo \citep{gupta2018shampoo}, and SOAP \citep{vyas2024soap} all make use of Kronecker factorizations to form a preconditioner.  Muon \citep{jordan2024muon,liu2025muon} is a steepest descent method in the matrix spectral norm.

\paragraph{Related variational objectives.}
With the mean held fixed, the covariance subproblem in batch-and-match variational inference \citep{cai2024} has the same form as the soft QN objective in~\Cref{eq:soft-qn-objective}, which is where the closed form in  \Cref{eq:qme-symmetric-main}
originates (see~\Cref{sec:variational-inference} for details). The least squares secant penalty within the soft QN objective~(\Cref{eq:soft-qn-objective})  has also been used to design preconditioned variants of SGD called PSGD in~\citep{li_preconditioned_2018,pooladzandi_curvature-informed_2024}. PSGD
then uses a form of multiplicative SGD applied to the  secant penalty to estimate the preconditioner.

\section{Experiments}
\label{sec:experiments}

We evaluate \SoftServe{} on recurrent sequence learning, image reconstruction, and physics-informed learning, from a small recurrent network to a 136M-parameter diffusion model. We test whether the structured inverse Hessian estimates improve training over first-order and other preconditioned methods. We compare training losses at equal gradient-evaluation budgets, and report prediction errors and wall-clock comparisons in Appendix~\ref{app:experiments}.

\paragraph{Comparison protocol.}
Our baselines are Adam, Muon, SOAP, K-FAC, and K-BFGS(L), with L-BFGS also included for deterministic training. To assess the contribution of the learned inverse Hessian, we include \SoftServe{}-I: the normalized momentum method~\Cref{eq:normalization} obtained by fixing $\mat{H}=\mat{I}$. We tune each method's learning rate separately and evaluate the selected configurations on three fresh seeds. All \SoftServe{} runs  refresh the inverse Hessian estimate every $K=10$ updates. Gradient counts include the additional evaluations used for secant pairs, line searches, and loss balancing. In stochastic training, \SoftServe{} uses one extra gradient evaluation every $K$ updates to form a secant pair on the same batch, see detailed pseudo-code in~\Cref{alg:softserve-kron}. Task-specific settings and learning-rate sweeps are given in Appendix~\ref{app:experiments}.

\subsection{RNN Adding}
\label{sec:rnn}

\setcounter{topnumber}{1}
\begin{figure}[t]
  \centering
  \includegraphics[width=\linewidth]{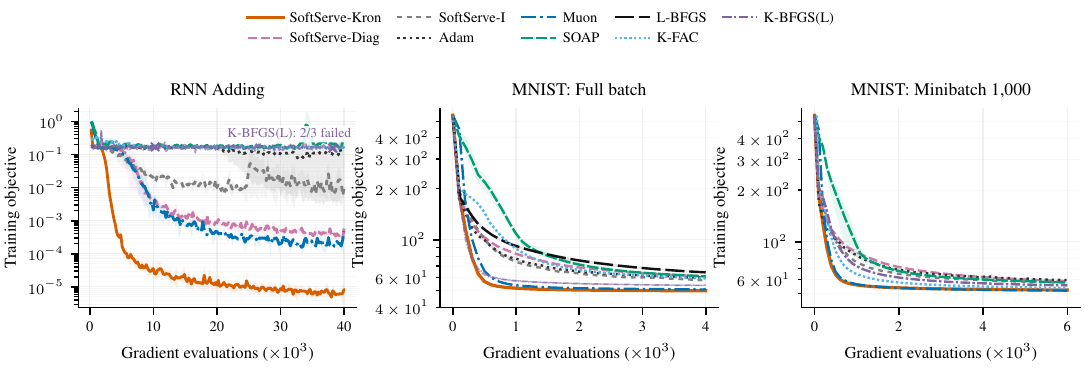}
  \caption{Training objectives against gradient evaluations: prediction MSE for RNN Adding (left), and reconstruction binary cross-entropy with $\ell_2$ regularization for full-batch MNIST (middle) and minibatch MNIST (right). Lines show three-seed means; shading spans the seed minimum and maximum. K-BFGS(L) with failed seeds are shown as individual traces.}
  \label{fig:rnn-main}
  \label{fig:autoencoder-main}
\end{figure}

\begin{figure}[t]
  \centering
  \includegraphics[width=\linewidth]{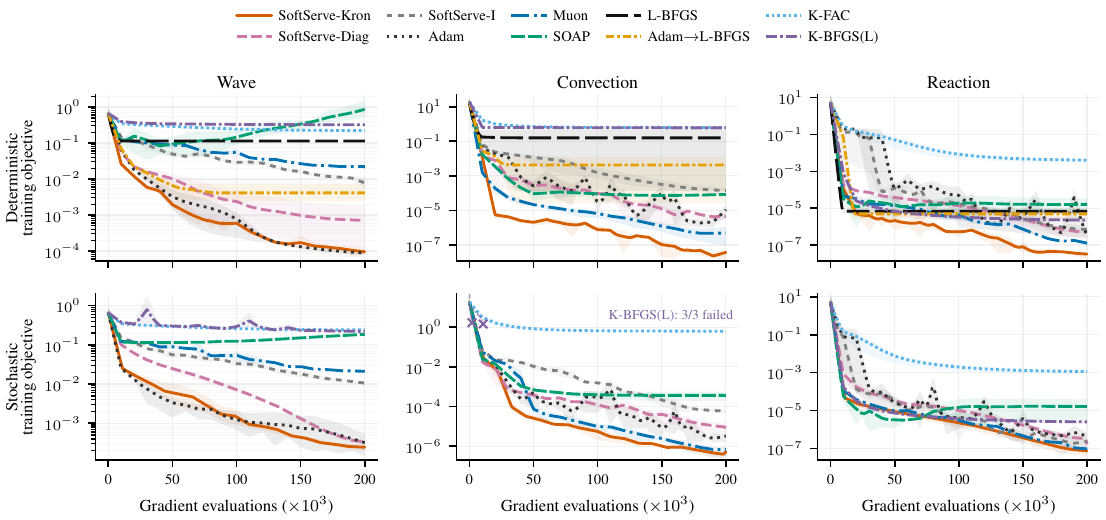}
  \caption{PINN training objectives for Wave, Convection, and Reaction, with fixed points (top) and resampled points (bottom). Lines show three-seed means; shading spans the seed minimum and maximum. Settings with failed seeds are shown as individual traces. Horizontal extensions show the final loss after an L-BFGS run terminates.}
  \label{fig:pinn-main}
\end{figure}

\begin{figure}[t]
  \centering
  \includegraphics[width=\linewidth]{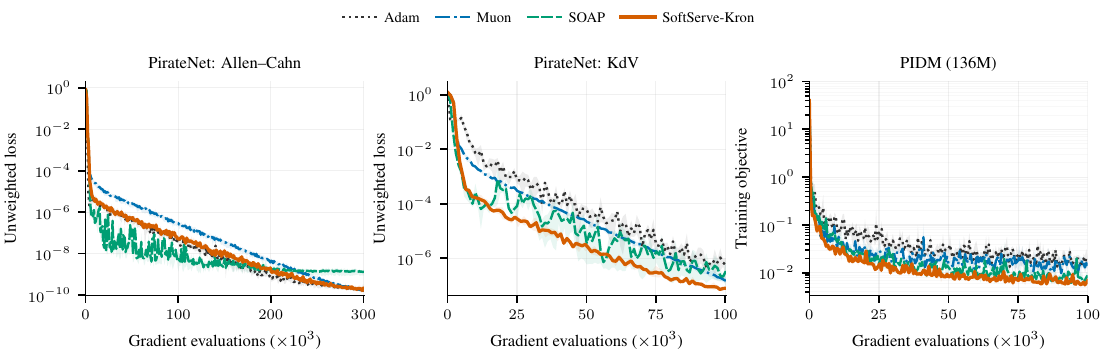}
  \caption{PirateNet on Allen--Cahn (left) and KdV (middle), and the 136M-parameter PIDM (right). PirateNet uses unweighted PDE and initial-condition MSE on fixed points; PIDM uses the training objective on fixed examples and noise. Lines show three-seed means; shading spans the seed minimum and maximum.}
  \label{fig:piratenet-main}
  \label{fig:pidm-main}
\end{figure}

The Adding problem tests a recurrent network's ability to retain information over a long sequence. Each input contains 100 numbers sampled uniformly from $[0,1]$ and markers identifying two positions. The network reads the full sequence and predicts the sum of the two marked numbers. We train a tanh RNN with 16 hidden units by minimizing mean squared error (MSE) on batches of 128 newly sampled sequences.
Although the network is small, repeated applications of the recurrent map can cause gradients to vanish or explode \citep{pascanu2013difficulty}. The full loss Hessian in this example is numerically nearly singular at initialization across four seeds (Appendix~\ref{app:rnn}). This makes Adding an interesting ill-conditioned small-scale test.
 Indeed, in \Cref{fig:rnn-main} we see that preconditioning was necessary to reach a reasonable loss, with 
\SoftServe{}-Kron reaching by far the lowest training loss (\Cref{fig:rnn-main}, left), outperforming our baselines on all three seeds.
\SoftServe{}-Kron mean test MSE is $6.6\times10^{-6}$, compared with $1.5\times10^{-4}$ for Muon and $1.0\times10^{-2}$ for \SoftServe{}-I. 

\subsection{MNIST autoencoder}
\label{sec:autoencoder}

We next consider MNIST reconstruction, a standard benchmark for second-order optimization \citep{hinton2006reducing,martens2010hessianfree,goldfarb2020practical}. A 2.8M-parameter feed-forward autoencoder compresses each $28\times28$ image into a 30-dimensional representation and reconstructs the input. We minimize pixelwise binary cross-entropy with $\ell_2$ regularization, using either all 60,000 training images at each update or minibatches of 1,000. These two settings compare deterministic and stochastic training of the same network. 

\SoftServe{}-Kron reaches the lowest training objective with full batches and is close to Muon with minibatches (\Cref{fig:autoencoder-main}, middle, right). It improves on \SoftServe{}-I in both settings. Appendix~\ref{app:autoencoder} shows held-out reconstruction errors and wall-clock comparisons.

\subsection{Physics-informed neural networks}
\label{sec:pinn}

Physics-informed neural networks (PINNs) approximate solutions to partial differential equations (PDEs) by mapping space--time coordinates to solution values. They are trained by penalizing the PDE residual and errors in the initial and boundary conditions. Depending on the differential operator, this loss can be very ill-conditioned, motivating the use of second-order optimization~\citep{rathore2024}. We evaluate \SoftServe{} on both small feed-forward PINNs and the larger residual architecture PirateNet \citep{wang2024piratenets}.

\paragraph{Small PINNs.}
We use the Wave, Convection, and Reaction problems of \citet{rathore2024}. Each network has three hidden layers of width 200 and 81,201 parameters.  Training minimizes the sum of the mean squared residuals of the governing PDE and its associated initial and boundary conditions. We either fix the sampled training coordinates (collocation points) or resample them at each update. This gives six settings across the three equations.

\SoftServe{}-Kron attains the lowest final training objective in five of the six settings, including all three with resampled points (\Cref{fig:pinn-main}). On deterministic Wave, Adam and Adam+L-BFGS attain lower final losses, with Adam+L-BFGS achieving the lowest. Adam+L-BFGS, used by \citet{rathore2024}, runs Adam for the first 11,000 parameter updates, followed by L-BFGS within the remaining gradient-evaluation budget. Appendix~\ref{app:pinn} gives solution errors and repeats the comparison in double precision.

For our remaining experiments with larger models below, we restrict the baselines to Adam, Muon, and SOAP, based on their scalability, and performance in the preceding experiments.

\paragraph{PirateNet.}
PirateNet is a PINN architecture that combines fully connected layers with skip connections to make deeper networks easier to train. We use a 0.73M-parameter model to solve the Allen--Cahn and Korteweg--de Vries (KdV) equations. Training uses 32,768 collocation points per update and adaptive weights for the loss terms. Since these weights differ across optimizers, we compare the unweighted sum of PDE and initial-condition MSE on the same fixed points.
Against Adam, Muon, and SOAP, \SoftServe{}-Kron is competitive with Muon on Allen--Cahn and reaches the lowest mean loss on KdV (\Cref{fig:piratenet-main}, left and middle). The PDEs and solution errors are given in Appendix~\ref{app:piratenet}.

\subsection{Physics-informed diffusion}
\label{sec:pidm}

To test \SoftServe{} at a larger scale, we train the physics-informed diffusion model (PIDM) of \citet{bastek2025physics} for structural design. Given applied forces, supports, and a material budget, the model learns to place material on a two-dimensional grid to form a stiff structure. We use the 
136M-parameter U-Net and 30,000 optimized training structures. The network reconstructs material-density and displacement fields from noisy inputs. Its objective combines denoising error with penalties for force imbalance, material-volume error, and compliance, a measure of structural flexibility.
At 100,000 gradient evaluations, \SoftServe{}-Kron reaches a lower mean training loss than Adam, Muon, and SOAP (\Cref{fig:pidm-main}, right).

\subsection{Language model}
\label{sec:llm}

Finally, we evaluate a 16M-parameter GPT-style language model \citep{radford2019language}. It has four layers, four attention heads per layer, hidden width 256, and a context length of 512 tokens. We minimize next-token cross-entropy on 200M tokens from FineWeb \citep{penedo2024fineweb}. \SoftServe{}-Kron reaches a mean validation NLL of $4.661$, lower than AdamW ($4.786$) but higher than Muon ($4.440$) and SOAP ($4.416$). Thus, the gains on the preceding tasks do not extend to matching Muon and SOAP on this language model.

\section{Discussion and Limitations}
\label{sec:discussion}

We introduce \SoftServe{},
a quasi-Newton approach for deep learning.
\SoftServe{} considers structured
Hessian approximations that are always positive definite,
GPU-friendly Newton--Schulz approximations, allowing it to scale to massive neural networks.
Our experiments highlight the value of curvature on several ill-conditioned tasks, including
recurrent networks, autoencoders, and physics-informed models, including the 136M-parameter PIDM.

Many future directions remain.
For instance, currently we use a fixed $\lambda$ within each training run. Developing adaptive $\lambda$ would further optimize performance.
Finally,
characterizing convergence guarantees
under realistic settings remains an open problem to study.

\label{page:main-end}

\subsection*{Acknowledgments}
We thank Lawrence Saul for first suggesting that our work on variational inference could be used to develop a quasi-Newton method, and for helpful discussions about this project.


\bibliography{references}

@article{ren2021kronecker,
  title = {Kronecker-factored quasi-{Newton} methods for deep learning},
  author = {Ren, Yi and Bahamou, Achraf and Goldfarb, Donald},
  journal = {arXiv preprint arXiv:2102.06737v3},
  year = {2022},
}

@misc{tieleman2012rmsprop,
  title = {Lecture 6.5---{RMSProp}: Divide the Gradient by a Running Average of Its Recent Magnitude},
  author = {Tieleman, Tijmen and Hinton, Geoffrey},
  howpublished = {Coursera: Neural Networks for Machine Learning},
  year = {2012},
}

@article{bordes2009sgd,
  title = {{SGD-QN}: Careful Quasi-{N}ewton stochastic gradient descent},
  author = {Bordes, Antoine and Bottou, L{\'e}on and Gallinari, Patrick},
  journal = {Journal of Machine Learning Research},
  volume = {10},
  number = {59},
  pages = {1737--1754},
  year = {2009},
}

@book{lancaster1995algebraic,
  title = {Algebraic Riccati Equations},
  author = {Lancaster, Peter and Rodman, Leiba},
  publisher = {Oxford University Press},
  year = {1995},
}

@article{svdnoGPU,
  title = {Efficient algorithms for computing rank-revealing factorizations on a {GPU}},
  author = {Heavner, Nathan and Chen, Chao and Gopal, Abinand and Martinsson, Per-Gunnar},
  journal = {Numerical Linear Algebra with Applications},
  volume = {30},
  number = {6},
  pages = {e2515},
  year = {2023},
}

@inproceedings{goldfarb2020practical,
  title = {Practical quasi-{N}ewton methods for training deep neural networks},
  author = {Goldfarb, Donald and Ren, Yi and Bahamou, Achraf},
  booktitle = {Advances in Neural Information Processing Systems (NeurIPS)},
  year = {2020},
}

@inproceedings{modi2024batch,
  title = {Batch, match, and patch: low-rank approximations for score-based variational inference},
  author = {Modi, Chirag and Cai, Diana and Saul, Lawrence K},
  booktitle = {International Conference on Artificial Intelligence and Statistics (AISTATS)},
  year = {2025},
}

@inproceedings{cai2024,
  title = {Batch and match: black-box variational inference with a score-based divergence},
  author = {Cai, Diana and Modi, Chirag and Pillaud-Vivien, Loucas and Margossian, Charles and Gower, Robert M. and Blei, David and Saul, Lawrence K.},
  booktitle = {International Conference on Machine Learning (ICML)},
  year = {2024},
}

@book{nocedal2006numerical,
  title = {Numerical {O}ptimization},
  author = {Nocedal, Jorge and Wright, Stephen J},
  publisher = {Springer},
  edition = {2nd},
  year = {2006},
}

@article{berglund2024,
  title = {Soft quasi-{Newton}: guaranteed positive definiteness by relaxing the secant constraint},
  author = {Berglund, Erik and Zhang, Jiaojiao and Johansson, Mikael},
  journal = {Optimization Methods and Software},
  volume = {40},
  number = {4},
  pages = {783--812},
  year = {2025},
}

@article{FletcherKL,
  title = {A New Variational Result for Quasi-{N}ewton Formulae},
  author = {Fletcher, R.},
  journal = {SIAM Journal on Optimization},
  volume = {1},
  number = {1},
  pages = {18--21},
  year = {1991},
}

@article{broyden1965convergence,
  title = {The Convergence of a Class of Double-rank Minimization Algorithms: 2. The New Algorithm},
  author = {Broyden, C. G.},
  journal = {IMA Journal of Applied Mathematics},
  volume = {6},
  number = {3},
  pages = {222--231},
  year = {1970},
}

@article{fletcher1970new,
  title = {A New Approach to Variable Metric Algorithms},
  author = {Fletcher, R.},
  journal = {The Computer Journal},
  volume = {13},
  number = {3},
  pages = {317--322},
  year = {1970},
}

@article{goldfarb1970family,
  title = {A Family of Variable-metric Methods Derived by Variational Means},
  author = {Goldfarb, Donald},
  journal = {Mathematics of Computation},
  volume = {24},
  number = {109},
  pages = {23--26},
  year = {1970},
}

@article{shanno1971conditioning,
  title = {Conditioning of Quasi-{Newton} Methods for Function Minimization},
  author = {Shanno, D. F.},
  journal = {Mathematics of Computation},
  volume = {24},
  number = {111},
  pages = {647--656},
  year = {1970},
}

@article{liu1989limited,
  title = {On the Limited Memory {BFGS} Method for Large Scale Optimization},
  author = {Liu, Dong C. and Nocedal, Jorge},
  journal = {Mathematical Programming},
  volume = {45},
  number = {1-3},
  pages = {503--528},
  year = {1989},
}

@book{horn2012matrix,
  title = {Matrix {A}nalysis},
  author = {Horn, Roger A and Johnson, Charles R},
  publisher = {Cambridge University Press},
  edition = {2nd},
  year = {2012},
}

@inproceedings{rathore2024,
  title = {Challenges in Training {PINN}s: A Loss Landscape Perspective},
  author = {Rathore, Pratik and Lei, Weimu and Frangella, Zachary and Lu, Lu and Udell, Madeleine},
  booktitle = {International Conference on Machine Learning (ICML)},
  year = {2024},
}

@inproceedings{gupta2018shampoo,
  title = {{Shampoo}: Preconditioned Stochastic Tensor Optimization},
  author = {Gupta, Vineet and Koren, Tomer and Singer, Yoram},
  booktitle = {International Conference on Machine Learning (ICML)},
  year = {2018},
}

@misc{jordan2024muon,
  title = {Muon: An optimizer for hidden layers in neural networks},
  author = {Jordan, Keller and Jin, Yuchen and Boza, Vlado and You, Jiacheng and Cesista, Franz and Newhouse, Laker and Bernstein, Jeremy},
  howpublished = {Keller Jordan blog},
  year = {2024},
}

@article{hinton2006reducing,
  title = {Reducing the dimensionality of data with neural networks},
  author = {Hinton, Geoffrey E. and Salakhutdinov, Ruslan R.},
  journal = {Science},
  volume = {313},
  number = {5786},
  pages = {504--507},
  year = {2006},
}

@inproceedings{martens2010hessianfree,
  title = {Deep learning via {Hessian}-free optimization},
  author = {Martens, James},
  booktitle = {International Conference on Machine Learning (ICML)},
  year = {2010},
}

@inproceedings{berahas_multi-batch_2016,
  title = {A Multi-Batch {L}-{BFGS} Method for Machine Learning},
  author = {Berahas, Albert S and Nocedal, Jorge and Takac, Martin},
  booktitle = {Advances in Neural Information Processing Systems (NeurIPS)},
  year = {2016},
}

@article{byrd_stochastic_2016,
  title = {A Stochastic Quasi-{Newton} Method for Large-Scale Optimization},
  author = {Byrd, Richard H. and Hansen, S. L. and Nocedal, Jorge and Singer, Yoram},
  journal = {SIAM Journal on Optimization},
  volume = {26},
  number = {2},
  pages = {1008--1031},
  year = {2016},
}

@inproceedings{scion,
  title = {Training deep learning models with norm-constrained {LMOs}},
  author = {Pethick, Thomas and Xie, Wanyun and Antonakopoulos, Kimon and Zhu, Zhenyu and Silveti-Falls, Antonio and Cevher, Volkan},
  booktitle = {International Conference on Machine Learning (ICML)},
  year = {2025},
}

@inproceedings{crawshaw2025muonvariants,
  title = {An Exploration of Non-{Euclidean} Gradient Descent: {Muon} and its Many Variants},
  author = {Michael Crawshaw and Chirag Modi and Mingrui Liu and Robert M. Gower},
  booktitle = {International Conference on Machine Learning (ICML)},
  year = {2026},
}

@article{dennis_jr_quasi-newton_1977,
  title = {Quasi-{Newton} {Methods}, {Motivation} and {Theory}},
  author = {Dennis, Jr., J. E. and Moré, Jorge J.},
  journal = {SIAM Review},
  volume = {19},
  number = {1},
  pages = {46--89},
  year = {1977},
}

@article{duchi_adaptive_2011,
  title = {Adaptive {Subgradient} {Methods} for {Online} {Learning} and {Stochastic} {Optimization}},
  author = {Duchi, John and Hazan, Elad and Singer, Yoram},
  journal = {Journal of Machine Learning Research},
  volume = {12},
  number = {61},
  pages = {2121--2159},
  year = {2011},
}

@article{higham_stable_1997,
  title = {Stable Iterations for the Matrix Square Root},
  author = {Nicholas J. Higham},
  journal = {Numerical Algorithms},
  volume = {15},
  number = {2},
  pages = {227--242},
  year = {1997},
}

@book{higham_functions_2008,
  title = {Functions of {Matrices}: {Theory} and {Computation}},
  author = {Higham, Nicholas J},
  publisher = {Society for Industrial and Applied Mathematics},
  year = {2008},
}

@article{irwin_secant_2020,
  title = {Secant penalized {BFGS}: a noise robust quasi-{Newton} method via penalizing the secant condition},
  author = {Irwin, Brian and Haber, Eldad},
  journal = {Computational Optimization and Applications},
  volume = {84},
  number = {3},
  pages = {651--702},
  year = {2023},
}

@inproceedings{kingma_adam_2015,
  title = {Adam: A Method for Stochastic Optimization},
  author = {Kingma, Diederik P. and Ba, Jimmy},
  booktitle = {International Conference on Learning Representations (ICLR)},
  year = {2015},
}

@article{li_preconditioned_2018,
  title = {Preconditioned {Stochastic} {Gradient} {Descent}},
  author = {Li, Xi-Lin},
  journal = {IEEE Transactions on Neural Networks and Learning Systems},
  volume = {29},
  number = {5},
  pages = {1454--1466},
  year = {2018},
}

@inproceedings{PSGD-Lie,
  title = {Preconditioner on Matrix {Lie} Group for {SGD}},
  author = {Li, Xi-Lin},
  booktitle = {International Conference on Learning Representations (ICLR)},
  year = {2019},
}

@article{liu2025muon,
  title = {Muon is Scalable for {LLM} Training},
  author = {Liu, Jingyuan and Su, Jianlin and Yao, Xingcheng and Jiang, Zhejun and others},
  journal = {arXiv preprint arXiv:2502.16982},
  year = {2025},
}

@article{Wolfe1969,
author = {Wolfe, Philip},
title = {Convergence Conditions for Ascent Methods},
journal = {SIAM Review},
volume = {11},
number = {2},
pages = {226-235},
year = {1969},
doi = {10.1137/1011036}
}

@inproceedings{kfac,
  title = {Optimizing Neural Networks with {Kronecker}-factored Approximate Curvature},
  author = {Martens, James and Grosse, Roger},
  booktitle = {International Conference on Machine Learning (ICML)},
  year = {2015},
}

@article{pooladzandi_curvature-informed_2024,
  title = {Curvature-Informed {SGD} via General Purpose {Lie}-Group Preconditioners},
  author = {Pooladzandi, Omead and Li, Xi-Lin},
  journal = {arXiv preprint arXiv:2402.04553},
  year = {2024},
}

@inproceedings{schraudolph_stochastic_2007,
  title = {A {Stochastic} {Quasi}-{Newton} {Method} for {Online} {Convex} {Optimization}},
  author = {Schraudolph, Nicol N. and Yu, Jin and G{\"u}nter, Simon},
  booktitle = {International Conference on Artificial Intelligence and Statistics (AISTATS)},
  year = {2007},
}

@article{schulz_iterative_1933,
  title = {Iterative {Berechnung} der reziproken {Matrix}},
  author = {Schulz, Günther},
  journal = {ZAMM - Zeitschrift für Angewandte Mathematik und Mechanik},
  volume = {13},
  number = {1},
  pages = {57--59},
  year = {1933},
}

@inproceedings{polarexpress,
  title = {The {Polar Express}: Optimal Matrix Sign Methods and their Application to the {Muon} Algorithm},
  author = {Amsel, Noah and Persson, David and Musco, Christopher and Gower, Robert M.},
  booktitle = {International Conference on Learning Representations (ICLR)},
  year = {2026},
}

@inproceedings{vyas2024soap,
  title = {{SOAP}: Improving and Stabilizing {Shampoo} using {Adam} for Language Modeling},
  author = {Vyas, Nikhil and Morwani, Depen and Zhao, Rosie and Shapira, Itai and Brandfonbrener, David and Janson, Lucas and Kakade, Sham},
  booktitle = {International Conference on Learning Representations (ICLR)},
  year = {2025},
}

@article{wang_stochastic_2017,
  title = {Stochastic Quasi-{Newton} Methods for Nonconvex Stochastic Optimization},
  author = {Wang, Xiao and Ma, Shiqian and Goldfarb, Donald and Liu, Wei},
  journal = {SIAM Journal on Optimization},
  volume = {27},
  number = {2},
  pages = {927--956},
  year = {2017},
}

@inproceedings{bastek2025physics,
  title = {Physics-Informed Diffusion Models},
  author = {Bastek, Jan-Hendrik and Sun, WaiChing and Kochmann, Dennis M.},
  booktitle = {International Conference on Learning Representations (ICLR)},
  year = {2025},
}

@inproceedings{xu2025fp64,
  title = {{FP64} is All You Need: Rethinking Failure Modes in Physics-Informed Neural Networks},
  author = {Xu, Chenhui and Liu, Dancheng and Nassereldine, Amir and Xiong, Jinjun},
  booktitle = {Advances in Neural Information Processing Systems (NeurIPS)},
  year = {2025},
}

@article{wang2024piratenets,
  title = {{PirateNets}: Physics-informed Deep Learning with Residual Adaptive Networks},
  author = {Wang, Sifan and Li, Bowen and Chen, Yuhan and Perdikaris, Paris},
  journal = {Journal of Machine Learning Research},
  volume = {25},
  number = {402},
  pages = {1--51},
  year = {2024},
}

@inproceedings{wang2025gradient,
  title = {Gradient Alignment in Physics-informed Neural Networks: A Second-Order Optimization Perspective},
  author = {Wang, Sifan and Bhartari, Ananyae Kumar and Li, Bowen and Perdikaris, Paris},
  booktitle = {Advances in Neural Information Processing Systems (NeurIPS)},
  year = {2025},
}

@techreport{radford2019language,
  title = {Language Models are Unsupervised Multitask Learners},
  author = {Radford, Alec and Wu, Jeffrey and Child, Rewon and Luan, David and Amodei, Dario and Sutskever, Ilya},
  institution = {OpenAI},
  year = {2019},
}

@inproceedings{penedo2024fineweb,
  title = {The {FineWeb} Datasets: Decanting the Web for the Finest Text Data at Scale},
  author = {Penedo, Guilherme and Kydl\'{\i}\v{c}ek, Hynek and Ben allal, Loubna and Lozhkov, Anton and Mitchell, Margaret and Raffel, Colin and {Von Werra}, Leandro and Wolf, Thomas},
  booktitle = {Advances in Neural Information Processing Systems (NeurIPS)},
  year = {2024},
}

@inproceedings{pascanu2013difficulty,
  title = {On the difficulty of training recurrent neural networks},
  author = {Pascanu, Razvan and Mikolov, Tomas and Bengio, Yoshua},
  booktitle = {International Conference on Machine Learning (ICML)},
  year = {2013},
}
\bibliographystyle{iclr2027_conference}

\clearpage
\appendix

\numberwithin{equation}{section}
\numberwithin{figure}{section}
\numberwithin{table}{section}
\numberwithin{algorithm}{section}
\renewcommand{\theHequation}{\theHsection.\arabic{equation}}
\renewcommand{\theHfigure}{\theHsection.\arabic{figure}}
\renewcommand{\theHtable}{\theHsection.\arabic{table}}
\renewcommand{\theHalgorithm}{\theHsection.\arabic{algorithm}}
\etocdepthtag.toc{appendix}
\renewcommand{\contentsname}{Appendix contents}
\etocsettagdepth{main}{none}
\etocsettagdepth{appendix}{all}
\etocsetnexttocdepth{subsection}
\tableofcontents

\clearpage

\section{BFGS and Soft quasi-Newton}

\subsection{BFGS}
\label{app:bfgs}

Typically, the BFGS update
is derived via a weighted Frobenius norm objective
that seeks to find the closest symmetric matrix
to the current iterate's quasi-Newton matrix
such that the secant equation is satisfied \citep[Chapter~6]{nocedal2006numerical}.
That is,
\begin{align}
\label{eq:frobnorm}
   \mat{H}_{t+1}= 
   \argmin_{\mat{H} \in \Splusplus} \|\mat{H} -\mat{H}_t\|_{F(\mat{W})} \quad \text{~subject~to~}\quad  \mat{H}=\mat{H}^\top, \quad \mat{H} y_t = s_t,
\end{align}
where the norm above is the weighted
Frobenius norm
$\|\mat{M}\|_{F(\mat{W})} = \norm[\mathrm F]{\mat{W}^\frac{1}{2} \mat{M} \mat{W}^\frac{1}{2}}$  
with a weighting matrix~$\mat{W}$
that satisfies the secant equation (e.g., $\mat{W} = \mat{B}_t$).

The BFGS update is given by
\begin{align}
\boxed{
    \mat{H}_{t+1}^{\text{BFGS}} 
    = \left(\mat{I} - \frac{s_t y_t^\top}{y_t^\top s_t}\right) \mat{H}_t
    \left(\mat{I} - \frac{y_t s_t^\top}{y_t^\top s_t}\right)
    +
\frac{s_t s_t^\top}{y_t^\top s_t}.
}
\label{eq:BFGS}
\end{align}

The log-determinant formulation of 
\citet{FletcherKL}
given in
 \Cref{eq:bfgs-inverse-variational}
recovers the same update, providing the starting point for replacing the secant equation with a soft penalty.

\subsection{Soft quasi-Newton}
\label{app:soft-quasi-newton}
As discussed in the main text,
Soft QN \citep{berglund2024}
then considers the soft penalty, regularized by the log-determinant divergence:
the Soft QN algorithm then approximates the inverse Hessian by solving
\begin{align}
\label{eq:softqn_objective_func}
\mat{H}_{t+1}=\underset{\mat{H}\in\Splusplus^D}{\arg\min}~
       \norm[\mat{H}^{-1}]{s_t-\mat{H}y_t}^2
 + \tfrac{1}{\lambda}\D(\mat{H}_t; \mat{H}).
\end{align}
\citet{berglund2024}
show that this objective
has the 
unique, positive definite
solution
\begin{equation}
\boxed{
 \mat{H}_{t+1}=
 \mat{V}_t
 -\frac{4\lambda \mat{V}_t \, y_ty_t^\top \, \mat{V}_t}
 {\big(1+\sqrt{1+4\lambda y_t^\top \mat{V}_ty_t}\big)^2}, 
 }
 \label{eq:dense-soft-qn-closed-form}
\end{equation}
 where $\mat{V}_t := \mat{H}_t + \lambda s_t s_t^\top$.
 Furthermore,
\citet[Theorem~3.2]{berglund2024} shows that
 the BFGS update is recovered
as a limiting case:
under positive curvature 
\(s_t^\top y_t>0\), 
then
\begin{align}
\boxed{
\mat{H}_{t+1}\to\mat{H}_{t+1}^{\rm BFGS},
\qquad \text{as}\quad \lambda\to\infty.}
\end{align}

\subsection{Connections with Variational inference}
\label{sec:variational-inference}
For Gaussians with a common mean \(\mu\), the log-determinant divergence
defined in \Cref{eq:ld-divergence} equals twice the KL divergence, that is
\[
\mathscr D(\mat{P};\mat{Q})=2\,\mathrm{KL}(\mathcal N(\mu,\mat{P})\,\|\,\mathcal N(\mu,\mat{Q})).
\]
Let \(\pi=\mathcal N(\mu,\mat{H})\), \(\pi_t=\mathcal N(\mu,\mat{H}_t)\), and let \(p\) be a differentiable positive density. With \(s=\mu-z\) and \(y=\nabla_z\log p(z)\), the Gaussian score is \(\nabla_z\log\pi(z)=\mat{H}^{-1}s\), so
\begin{equation}
 \norm[\mat{H}]{\nabla_z\log\pi(z)-y}^2=\norm[\mat{H}^{-1}]{s-\mat{H}y}^2,\qquad
 \tfrac{2}{\lambda}\mathrm{KL}(\pi_t\|\pi)=\tfrac{1}{\lambda}\mathscr D(\mat{H}_t; \mat{H}).
 \label{eq:score-substitution}
\end{equation}
Thus, the objective in ~\Cref{eq:softqn_objective_func} is the one-score covariance objective of \citet{cai2024} with the mean fixed. The identity links the two covariance subproblems; it does not equate variational inference with quasi-Newton optimization.

\section{\SoftServe{}: structured updates}
\label{app:proofs}

Recall that \SoftServe{}
solves the soft quasi-Newton objective
with respect to a structured matrix family~$\mathcal{H}$ (diagonal or Kronecker factored):
\begin{align} 
\label{eq:softserve-obj}
\boxed{
\mat{H}_{t+1}=\underset{\mat{H}\in\mathcal{H}}{\arg\min}~
       \norm[\mat{H}^{-1}]{s_t-\mat{H}y_t}^2
 + \tfrac{1}{\lambda}\D(\mat{H}_t; \mat{H}).
 }
\end{align}


\subsection{Diagonal update}

We begin by considering the diagonal family
\begin{align*}
\mathcal H_{\rm Diag} &=\{\diag(h):h\in\R_{++}^D\}.
\end{align*}

The lemma below derives the update.

\begin{restatable}[Exact diagonal update]{lemma}{diagupdatelemma}
\label{lem:diag-update}
Let \(h_t\in\R_{++}^D\), \(\mat{H}_t=\diag(h_t)\), \(s_t,y_t\in\R^D\), and \(\lambda>0\). The diagonal restriction of~\Cref{eq:softserve-obj} has the unique minimizer \(\mat{H}_{t+1}=\diag(h_{t+1})\), where
\begin{equation} \label{eq:diag-update}
 h_{t+1,i}=
 \frac{2\left(h_{t,i}+\lambda s_{t,i}^2\right)}
 {1+\sqrt{1+4\lambda y_{t,i}^2
                  \left(h_{t,i}+\lambda s_{t,i}^2\right)}},
 \qquad i=1,\ldots,D.
\end{equation}
Hence \(h_{t+1}\in\R_{++}^D\), regardless of the signs of \(s_{t,i}y_{t,i}\).
\end{restatable}

\begin{proof}
For \(\mat{H}=\diag(h)\), the terms of \Cref{eq:softserve-obj} that depend on \(h\) are
\[
 \sum_{i=1}^D\left[
 \frac{s_{t,i}^2}{h_i}+h_iy_{t,i}^2+
 \frac{1}{\lambda}\left(\frac{h_{t,i}}{h_i}+\log h_i\right)\right].
\]
Fix \(i\), set \(v=h_{t,i}+\lambda s_{t,i}^2>0\) and \(b=h_i^{-1}\), and multiply by \(\lambda\). The scalar problem is
\[
\begin{array}{cc}
 \underset{b>0}{\mbox{minimize}} &
\psi(b),\qquad
 \psi(b)=vb+\frac{\lambda y_{t,i}^2}{b}-\log b.
\end{array}
\]
Since \(\psi''(b)=2\lambda y_{t,i}^2/b^3+1/b^2>0\) and
\(\psi(b)\to+\infty\) as \(b\downarrow0\) or \(b\to\infty\),
the minimizer exists and is unique.
The stationarity condition \(\lambda y_{t,i}^2h_i^2+h_i-v=0\) has the single positive root
\[
 h_i=\frac{2v}{1+\sqrt{1+4\lambda y_{t,i}^2v}},
\]
which also covers \(y_{t,i}=0\). Separability over \(i\) proves the claim.
\end{proof}

\subsection{Kronecker update}
\label{asec:kronecker}

We state standard identities for Kronecker products \citep{horn2012matrix}.
\begin{lemma}
For \(\mat{P}\in\R^{m\times m}\), \(\mat{Q}\in\R^{n\times n}\), and conformable \(\mat{X},\mat{Z}\), we have:
\begin{align}
 (\mat{P}_1\otimes \mat{Q}_1)(\mat{P}_2\otimes \mat{Q}_2)&=(\mat{P}_1\mat{P}_2)\otimes(\mat{Q}_1\mat{Q}_2) \label{eq:kron-mixed-product}
 \\
 (\mat{P}\otimes \mat{Q})\vecop(\mat{X})&=\vecop(\mat{Q}\mat{X}\mat{P}^\top)\label{eq:kron-vec}\\
 \vecop(\mat{X})^\top\vecop(\mat{Z})&=\tr(\mat{X}^\top \mat{Z})\label{eq:vec-inner-product}
 \\
 \tr(\mat{P}\otimes \mat{Q})&=\tr(\mat{P})\tr(\mat{Q})\label{eq:kron-trace}
 \\
 \det(\mat{P}\otimes \mat{Q}) &=\det(\mat{P})^n\det(\mat{Q})^m.
 \label{eq:kron-det}
\end{align}
\end{lemma}

We now consider the Kronecker-structured
family:
\begin{align}
 \mathcal H_{\rm Kron} &=\{\mat{L}\otimes \mat{R}:\mat{L}\in\mathbb S_{++}^m,\mat{R}\in\mathbb S_{++}^n\}.
\end{align}
Here, we derive the alternating update from \Cref{sec:kron}.
First, we recall some notation.
Write \(s_t=\vecop(\mat{S}_t)\), \(y_t=\vecop(\mat{Y}_t)\), and
\(\mat{H}_t=\mat{L}_t\otimes\mat{R}_t\).
The following result establishes
the properties of 
these conditional updates.

\begin{lemma}[Alternating Kronecker solution.]
   \label{lem:kron-sweep}
Let \(\mat{L}_t\in\mathbb S_{++}^m\), \(\mat{R}_t\in\mathbb S_{++}^n\), \(\lambda>0\), and \(\mat{S}_t,\mat{Y}_t\in\R^{n \times m}\). Consider the objective 
given in~\Cref{eq:soft-qn-objective}
and the updates
    \begin{align}
\mat{L}_{t+1} &=
    \argmin_{\mat{L} \in \Splusplus^m} \quad\tr(\mat{U}_L \mat{L}) + \tr(\mat{V}_L \mat{L}^{-1}) + \log\det \mat{L}\label{eq:kron-l-subproblem}
\\
\mat{R}_{t+1} &=
    \argmin_{\mat{R} \in \Splusplus^n} \quad\tr(\mat{U}_R \mat{R}) + \tr(\mat{V}_R \mat{R}^{-1}) + \log\det\mat{R},\label{eq:kron-r-subproblem}
\end{align}
where we define the
auxiliary matrices:
\begin{align} 
    \mat{U}_L &:= \tfrac{\lambda}{n} \mat{Y}_t^\top \mat{R}_t \mat{Y}_t \quad \mbox{~~and} \quad \mat{V}_L := \mat{L}_t + \tfrac{\lambda}{n} \mat{S}_t^\top \mat{R}_t^{-1} \mat{S}_t\label{eq:ULVL-app}
    \\
    \mat{U}_R &:= \tfrac{\lambda}{m} \mat{Y}_t \mat{L}_{t+1} \mat{Y}_t^\top \quad \mbox{and} 
\quad \mat{V}_R := m^{-1} [\tr(\mat{L}_t \mat{L}_{t+1}^{-1}) \mat{R}_t + \lambda \mat{S}_t \mat{L}_{t+1}^{-1} \mat{S}_t^\top].\label{eq:URVR-app}
\end{align}
These updates have a unique, closed-form solution that costs $\mathcal{O}(n^3+m^3+n^2m+nm^2)$ to compute exactly.
Furthermore, the updates to $\mat{L}$ and $\mat{R}$ monotonically decrease the soft QN objective.
\end{lemma}
\begin{proof}
Substituting 
$\mat{H} = \mat{L} \otimes \mat{R}$
into the 
objective in \Cref{eq:softserve-obj}
and applying the identities in \Crefrange{eq:kron-mixed-product}{eq:kron-det}, up to a constant independent of $\mat{L}$ and $\mat{R}$,
\begin{align}
 \Phi_t(\mat{L},\mat{R})
 =\;&\tr(\mat{L}\mat{Y}_t^\top\mat{R}\mat{Y}_t)
     +\tr(\mat{L}^{-1}\mat{S}_t^\top\mat{R}^{-1}\mat{S}_t)\notag\\
 &+\tfrac{1}{\lambda}\tr(\mat{L}_t\mat{L}^{-1})
                         \tr(\mat{R}_t\mat{R}^{-1})
  +\tfrac{1}{\lambda}\left[n\log\det\mat{L}+m\log\det\mat{R}\right].
 \label{eq:kron-objective}
\end{align}
This expression makes the
dependence on each factor explicit.
We minimize it first over $\mat{L}$
with $\mat{R}=\mat{R}_t$,
then over $\mat{R}$ with 
$\mat{L}=\mat{L}_{t+1}$ fixed.

Holding $\mat{R}=\mat{R}_t$ fixed and collecting  terms in~\Cref{eq:kron-objective} involving $\mat{L}$,
we obtain the conditional update
\begin{align}
\mat{L}_{t+1} =
    \argmin_{\mat{L} \in \Splusplus^m} \quad\tr(\mat{U}_L \mat{L}) + \tr(\mat{V}_L \mat{L}^{-1}) + \log\det \mat{L}.\label{eq:factorLupdate2}
\end{align}
Differentiating the objective in \Cref{eq:factorLupdate2} with respect to $\mat{L}$ and setting the derivative to zero yields the  quadratic matrix equation
\begin{align}
    \mat{L} \, \mat{U}_L \, \mat{L} + \mat{L} = \mat{V}_L.
    \label{eq:factorA}
\end{align}
Here, because $\mat{U}_L \succeq 0$ and $\mat{V}_L \succ 0$, this equation has a positive-definite solution~$\mat{L}_{t+1}$ (\Cref{lemma:qme}).
To establish uniqueness,
we appeal to the observation that the objective 
is strictly convex in
$\mat{L}^{-1}$.

We then update the factor $\mat{R}$ in an analogous manner, holding $\mat{L}_{t+1}$ fixed:
\begin{align}
\mat{R}_{t+1} =
    \argmin_{\mat{R} \in \Splusplus^n} \quad\tr(\mat{U}_R \mat{R}) + \tr(\mat{V}_R \mat{R}^{-1}) + \log\det\mat{R},
    \label{eq:factorRupdate2}
\end{align}
Solving
    \Cref{eq:factorRupdate2}
    for $\mat{R}_{t+1}$
    reduces to solving another
    quadratic matrix equation
\begin{align}
    \mat{R} \, \mat{U}_R \, \mat{R} + \mat{R} = \mat{V}_R,\label{eq:kron-r-qme}
\end{align} 
which again has a unique, positive-definite solution.

Each block update minimizes the objective over one factor,
holding the other factor and the reference matrix \(\mat{H}_t\)
fixed. Since the current factor is feasible, each update
cannot increase the objective.
Forming the coefficient products costs \(\O(n^2m+nm^2)\); the factor inverses and two quadratric matrix equation solves cost \(\O(n^3+m^3)\).
\end{proof}

\subsection{Quadratic Matrix Equations}
\label{app:qme}

In the previous~\Cref{asec:kronecker} we showed that \SoftServe{}-Kron
has an alternating update for the factors
$\mat{L}$
and $\mat{R}$, where first we perform
\begin{align}
    \boxed{
    \text{Update~$\mat{L}_{t+1}$~by~solving~}
    \mat{L} \, \mat{U}_L \, \mat{L} + \mat{L} = \mat{V}_L,
    }
   \label{eq:factorLupdate}
\end{align}
and then we perform the update
\begin{align}
    \boxed{
    \text{Update~$\mat{R}_{t+1}$~by~solving~}
    \mat{R} \, \mat{U}_R \, \mat{R} + \mat{R} = \mat{V}_R.
    }
    \label{eq:factorRupdate}
\end{align}

Given known matrices 
 $\mat{U} \succeq 0$ and $\mat{V} \succ 0$,
 each update above
 involves solving a quadratic matrix equation of the form
\begin{align}
    \boxed{
\mat{Q}\mat{U}\mat{Q}+\mat{Q}=\mat{V}.
    }
    \label{eq:qme-app}
\end{align}

\citet[Lemma~B.2]{cai2024}
    show that \Cref{eq:qme-symmetric-app} below is the positive definite solution.

\begin{lemma}[Quadratic matrix equation]
\label{lemma:qme}
    The positive definite solution to  \Cref{eq:qme-app} is given by
\begin{equation}
\mat{Q}
 =2\mat{V}^{1/2}
\left[\mat{I} + \left(\mat{I} + 4\mat{V}^{1/2}\mat{U}\mat{V}^{1/2}\right)^{1/2}\right]^{-1}
 \mat{V}^{1/2}.
 \label{eq:qme-symmetric-app}
\end{equation}
\end{lemma}

When updating either Kronecker factor, \(\mat{L}\) or \(\mat{R}\),
the matrix \(\mat{U}\) in \Cref{eq:qme-app} admits a Gram factorization
\(\mat{U}=\mat{F}\mat{F}^\top\).
\citet[Lemma~B.3]{cai2024} (see also \citet{modi2024batch})
provide an equivalent expression for the solution to the quadratic matrix equation
in \Cref{eq:qme-app}
in terms of this Gram matrix,
which is given by
\begin{align}
    \boxed{
    \mat{Q} = \mat{V} - \mat{V}^\top \mat{F}
    \left[
    \tfrac{1}{2} \mat{I} + (\mat{F}^\top \mat{V} \mat{F} + \tfrac{1}{4} \mat{I})^\frac{1}{2}
    \right]^{-2} \mat{F}^\top \mat{V}.
    }
    \label{eq:gram:app}
\end{align}
For a matrix $\mat{F} \in \reals^{d \times p}$,
this update takes $\O(d^2 p + p^3)$ computation.

Our implementation of \SoftServe{}-Kron,
which relies on the coupled Newton-Schulz,
will rely on the latter Gram
form of the solution in
 \Cref{eq:gram:app}.

\section{\SoftServe{}: numerical computation of Kronecker updates}

The Kronecker updates in
\Cref{sec:kron} require solving two
quadratic matrix equations,
one for each factor.
In this section, we describe 
how to evaluate these updates
using matrix multiplications and
Newton-Schulz iterations,
as outlined in \Cref{sec:gram}.
We begin by explaining how to form the coefficients $\mat{U}, \mat{V}$ of the quadratic matrix equations,
then we describe the Newton-Schulz procedures,
and finally we explain how to apply the Newton-Schulz procedures 
to obtain the full Gram update, providing explicit pseudocode
for the Kronecker factor updates $\mat{L}$ and $\mat{R}$.

\label{app:numerics}

\subsection{Evaluating the updates in Gram form}
\label{app:gram:details}
In this section, we show that the coefficients
$\mat{U} = \mat{F}\mat{F}^\top$ and $\mat{V}$ 
can be written in terms of matrix multiplications and square root and inverse square root operations.
Consider the Gram solution to the quadratic matrix equation in
\Cref{eq:gram:app}.
This expression requires a square root of 
$\mat{C} =  \mat{F}^\top \mat{V} \mat{F} + \tfrac{1}{4}\mat{I}$
and an inverse of 
$\tfrac{1}{2}\mat{I} + \mat{C}^\frac{1}{2}$.
Overall, this solution requires fewer matrix function evaluations (i.e., square root and inverse)
than
the solution in
\Cref{eq:qme-symmetric-app}.

We begin by motivating the Gram update further by examining the factorization \(\mat{U}=\mat{F}\mat{F}^\top\).
For our Kronecker factor updates,
we recall the definitions of $\mat{U}_L$
and $\mat{U}_R$ in \Cref{eq:ULVL-app,eq:URVR-app},
\begin{align}
    \mat{U}_L &= \tfrac{\lambda}{n} \mat{Y}_t^\top \mat{R}_t \mat{Y}_t \quad 
    \\
    \mat{U}_R &= \tfrac{\lambda}{m} \mat{Y}_t \mat{L}_{t+1} \mat{Y}_t^\top. 
\end{align}
Thus,
the required factorizations arise from
\begin{align}
    \mat{U}_L &= \mat{F}_L \mat{F}_L^\top,
    \qquad
    \mat{F}_L = \sqrt{\tfrac{\lambda}{n}} \mat{Y}_t^\top \mat{R}_t^\frac{1}{2} \\
    \mat{U}_R &= \mat{F}_R \mat{F}_R^\top,
    \qquad
    \mat{F}_R = \sqrt{\tfrac{\lambda}{m}} \mat{Y}_t \mat{L}_{t+1}^\frac{1}{2}. 
\end{align}

The remaining coefficients $\mat{V}_L$
and $\mat{V}_R$ can be
formed with the \emph{corresponding} inverse square roots to the square roots used above.
For example,
for $\mat{V}_L = \mat{L}_t + \tfrac{\lambda}{n} \mat{S}_t^\top \mat{R}_t^{-1} \mat{S}_t$, we can form the second term 
using
\begin{align}
    \mat{S}_t^\top \mat{R}_t^{-1}\mat{S}_t
    =
    (\mat{R}_t^{-\frac{1}{2}}\mat{S}_t)^\top
    (\mat{R}_t^{-\frac{1}{2}}\mat{S}_t).
\end{align}
This construction allows us to reuse both outputs of each coupled Newton-Schulz call:
the square root forms $\mat{F}_L$ (or $\mat{F}_R)$,
and the corresponding inverse square root forms
$\mat{V}_L$ (or $\mat{V}_R$),
avoiding a separate Newton-Schulz call to approximate the inverse $\mat{R}_t^{-1}$.

In
\Cref{alg:kron-gram}, we summarize how to form 
all of the coefficients in this manner.

\begin{algorithm}[t]
    \caption{Newton--Schulz routines}
    \label{alg:ns}
    \begin{algorithmic}[1]
    \Function{NewtonSchulzRootPair}{$\mat{M}, J$}
    \State $a \gets \|\mat{M}\|_F$
    \State $\mat{Z} \gets \mat{M}/a$
    \State $\mat{P} \gets \mat{I}$
    \For{$j=0,\ldots,J-1$}
        \State $\mat{T} \gets \tfrac{1}{2} (3 \mat{I} - \mat{P}\mat{Z})$
        \State $\mat{Z} \gets \operatorname{sym}[\mat{Z} \mat{T}]$
        \State $\mat{P} \gets \operatorname{sym}[\mat{T} \mat{P}]$
    \EndFor
    \State
    \Return $(\sqrt{a} \, \mat{Z}; \mat{P} /\sqrt{a})$
    \Comment{Approximations to $\mat{M}^{\frac{1}{2}}, \mat{M}^{-\frac{1}{2}}$}
    \EndFunction
    \Statex
    \Function{NewtonSchulzInverse}{$\mat{M}, J$}
        \State $\mat{E} \gets \mat{I} / \|\mat{M}\|_F$
        \For{$j = 0,\ldots,J-1$}
            \State $\mat{E} \gets \operatorname{sym}\!\left[\mat{E}(2 \mat{I}-\mat{M} \mat{E})\right]$
        \EndFor
        \State \Return $\mat{E}$
        \Comment{Approximation to $\mat{M}^{-1}$}
    \EndFunction
    \end{algorithmic}
\end{algorithm}

\begin{algorithm}[t]
    \caption{Kronecker updates using the Gram solution}
    \label{alg:kron-gram}
    \begin{algorithmic}[1]
    \Function{SolveGram}{$\mat{F}, \mat{V}, J_{\mathrm{root}}, J_{\mathrm{inv}}$}
    \Comment{Approximately solve
    $\mat{Q}\mat{F}\mat{F}^\top \mat{Q} + \mat{Q} = \mat{V}$}
    \State $\mat{C} \gets \mat{F}^\top \mat{V} \mat{F} +  \tfrac{1}{4}\mat{I}$
     \State $({\mat{B}}, \texttt{\_}) \gets$ \Call{NewtonSchulzRootPair}{$\mat{C}, J_{\text{root}}$}
     \State $\mat{W} \gets$ \Call{NewtonSchulzInverse}{$\tfrac{1}{2} \mat{I} + \mat{B}, J_{\mathrm{inv}}$}
     \State $\mat{E} \gets \mat{V}\mat{F}\mat{W}$
     \State \Return $\operatorname{sym}[\mat{V} - \mat{E}\mat{E}^\top]$
    \EndFunction

    \Statex
    
    \Function{KroneckerUpdate}{$\mat{L}, \mat{R}, \mat{S}, \mat{Y}, \lambda, J_{\mathrm{root}}, J_{\mathrm{inv}}$}
    \Statex
    \Comment{\red{Update $\mat{L}$, holding $\mat{R}$ fixed}}
    \State $(\mat{C}_R, \mat{D}_R) \gets$
    \Call{NewtonSchulzRootPair}{$\mat{R}, J_{\mathrm{root}}$}
    
    \State $\mat{F}_L \gets \sqrt{\tfrac{\lambda}{n}} \, \mat{Y}^\top \mat{C}_R$
    \Comment{Form factored coefficient $\mat{U}_L = \mat{F}_L\mat{F}_L^\top$}
    \State $\mat{Z}_R \gets \mat{D}_R \mat{S}$
    \State $\mat{V}_L \gets \mat{L} + \tfrac{\lambda}{n} \, \mat{Z}_R^\top \, \mat{Z}_R$
    \Comment{Form  coefficient $\mat{V}_L$} 
    \State $\mat{L}^+ \gets$ \Call{SolveGram}{$\mat{F}_L,\mat{V}_L, J_{\mathrm{root}}, J_{\mathrm{inv}}$}
    \Statex
    \Comment{\red{Update $\mat{R}$, holding $\mat{L}^+$ fixed}}
    \State $(\mat{C}_L, \mat{D}_L) \gets$
    \Call{NewtonSchulzRootPair}{$\mat{L}^+, J_{\mathrm{root}}$}
    \State $\mat{F}_R \gets \sqrt{\tfrac{\lambda}{m}} \, \mat{Y} \mat{C}_L$
    \Comment{Form factored coefficient $\mat{U}_R = \mat{F}_R\mat{F}_R^\top$}
    \State $\mat{Z}_L \gets \mat{S} \mat{D}_L $
    \State $\alpha_L \gets 
    \tfrac{1}{m} \, \tr(\mat{D}_L \mat{L} \mat{D}_L)$
   \Comment{Depends on  original factor $\mat{L}$}
    \State $\mat{V}_R \gets \alpha_L \mat{R}+ \tfrac{\lambda}{m} \, \mat{Z}_L \mat{Z}_L^\top$
    \Comment{Form  coefficient $\mat{V}_R$} 
    \State $\mat{R}^+ \gets$ \Call{SolveGram}{$\mat{F}_R,\mat{V}_R, J_{\mathrm{root}}, J_{\mathrm{inv}}$}
    \EndFunction
    \Return $(\mat{L}^+, \mat{R}^+)$
    \end{algorithmic}
\end{algorithm}

\subsection{Newton-Schulz for the square root and inverse}
\label{asec:newton-schulz-couple} 
To compute the Gram solution of the quadratic matrix equation, we need an approximation for the matrix square root
and the matrix inverse.
We approximate the required matrix functions using Newton-Schulz iterations, which construct polynomial approximations through repeated matrix multiplications and additions.
We summarize these routines in 
\Cref{alg:ns}
as \textsc{NewtonSchulzRootPair}
and
\textsc{NewtonSchulzInverse}, respectively.

For square roots,
we use the coupled Newton-Schulz iteration,
which simultaneously approximates a 
matrix square root and its inverse
using matrix multiplications \citep{higham_stable_1997,higham_functions_2008}.
For an input matrix
$\mat{M} \in \Splusplus^d$,
let $\|\mat{M}\|_F$ denote its Frobenius norm.
First, we rescale the matrix~$\|\mat{M}\|_F$, placing the eigenvalues between $0$ and $1$:
\begin{align}
a &= \|\mat{M}\|_F,  \\
\mat{Z}_0 &= \mat{M}/a, \\
\mat{P}_0 &= \mat{I}.
\end{align}

In the Newton-Schulz procedure,
we then iterate $j=0,\ldots,J-1$ times through
\begin{align}
  \mat{T}_j &= \tfrac{1}{2}(3 \mat{I} - \mat{P}_j \mat{Z}_j) \\
    \mat{Z}_{j+1} &= \text{sym}[\mat{Z}_j \mat{T}_j] \\
    \mat{P}_{j+1} &= \text{sym}[\mat{T}_j \mat{P}_j],
    \end{align}
 where we define the matrix symmetrization operation
\begin{align}
\label{eq:sym}
\text{sym}[\mat{A}] = (\mat{A}+\mat{A}^\top)/2.
\end{align}

The iterates
$\mat{Z}_j$ and $\mat{P}_j$ approximate the
square root and inverse square root, respectively,
of the scaled input.
The correction~$\mat{T}_j = \tfrac{1}{2}(3 \mat{I} - \mat{P}_j \mat{Z}_j)$
is a first-degree polynomial in $\mat{P}_j \mat{Z}_j$.

After $J$~iterations,
rescaling by $\sqrt{a}$ gives the desired approximations with respect to the original input matrix are given by
\begin{align}
    \mat{M}^\frac{1}{2} &\approx \sqrt{a} \, \mat{Z}_{J} \\
    \mat{M}^{-\frac{1}{2}} &\approx \mat{P}_J / \sqrt{a}.
\end{align}
To approximate the inverse of the matrix,
we apply a separate Newton-Schulz iteration~\citep{schulz_iterative_1933,higham_functions_2008}.
For a positive definite matrix $\mat{M}\in \Splusplus^d$,
we initialize and then iterate $j=0,\ldots,J-1$ times 
\begin{align}
    \mat{E}_0 &= \mat{I} / \|\mat{M} \|_F 
    \\
    \mat{E}_{j+1} &= \operatorname{sym}[\mat{E}_j (2 \mat{I} -\mat{M} \mat{E}_j)].
\end{align}
This update builds a polynomial approximation to
$\mat{M}^{-1}$.
After $J$ iterations, we return the approximation~$\mat{E}_J \approx \mat{M}^{-1}$.

\subsection{Computing the Kronecker updates via Newton-Schulz}

We now combine the Newton-Schulz routines to evaluate the Gram update in \Cref{eq:gram:app}.
Given $\mat{F}$ and $\mat{V}$, we form
$\mat{C} =  \mat{F}^\top \mat{V} \mat{F} + \tfrac{1}{4}\mat{I}$
and apply 
 \textsc{NewtonSchulzRootPair} to obtain an approximation
 $\widehat{\mat{B}} \approx \mat{C}^\frac{1}{2}$.
We then apply
\textsc{NewtonSchulzInverse}
to obtain $\widehat{\mat{W}} \approx (\tfrac{1}{2}\mat{I} + \widehat{\mat{B}})^{-1}$.
The resulting Newton-Schulz approximation to the 
Gram solution~\eqref{eq:gram:app} of the quadratic matrix equation
is then given by
\begin{align}
    \widehat{\mat{Q}} =
    \mat{V} - 
    (\mat{V}\mat{F}{\widehat{\mat{W}}})(\mat{V}\mat{F}{\widehat{\mat{W}}})^\top.
\end{align}
The resulting update only requires these two
Newton-Schulz routines and matrix multiplications.
\Cref{alg:kron-gram} presents the pseudocode for this Gram update and 
its application within  the alternating Kronecker update for the factors $\mat{L}$ and $\mat{R}$.
In our implementation, we use 18 iterations for the square root routines, and 10 for the inverse routine.

In the application of these procedures
to the quadratic
matrix equation solution,
our implementation symmetrizes every output
and, for every solve,
records the normalized residual \[\rho_Q=\norm[\mathrm F]{\mat{Q}\mat{U}\mat{Q}+\mat{Q}-\mat{V}}/\norm[\mathrm F]{\mat{V}},\]
so we can verify how close we are to the exact solution to the quadratic matrix equation.

\section{Optimizer Details}
\label{app:optimizer}

\subsection{Training with \SoftServe{}}
\label{app:training:details}

We provide more of the details of how our \SoftServe{} methods are implemented. In particular, this section will explain the detailed pseudo code of \SoftServe{}-Kron and \SoftServe{}-Diag is in~\Cref{alg:softserve-kron} and~\Cref{alg:softserve-diag}, respectively.

\SoftServe{}-Kron is a method for updating matrix parameters, where we will use $\mX\in \R^{n\times  m}$ to denote the matrix parameters of a given layer. Whereas \SoftServe{}-Diag is applicable to parameters of any shape, and thus in~\Cref{alg:softserve-diag}, we represent the given parameters as a vector $x \in \R^{d}$. 
If this corresponds to a linear layer, then $d = nm$, where $n\times m$ are the dimensions of the linear layer.

 To make \SoftServe{} practical, we will make use of some now-standard practices in deep learning, including the use of momentum,   updating the approximate Hessian every so often to amortize costs, and taking care to properly scale the update directions. We address each of these considerations in turn.

To estimate the gradient for the quasi-Newton step, we will use momentum estimate which in matrix form is updated
according to lines~(\ref{ln:mom1}--\ref{ln:mom2}) in~\Cref{alg:softserve-kron}, and in vector form is updated according to lines~(\ref{ln:mom1-d}--\ref{ln:mom2-d}) in~\Cref{alg:softserve-diag}.
We also incorporate an additional Nesterov extrapolation or   
the debiasing trick~\cite{kingma_adam_2015} to form the final momentum estimate.

Restricting the Hessian approximation to a diagonal or Kronecker factor greatly reduces the costs as compared to the dense Soft QN method. For instance, for a $n\times m$ linear layer, the state, application, and refresh costs of \SoftServe{} are given in~\Cref{tab:costs}.
\begin{table}  \centering
\begin{tabular}{lccc}
\toprule
Metric & State & Apply & Refresh \\
\midrule
Dense & \(\mathcal O(n^2m^2)\) & \(\mathcal O(n^2m^2)\) & \(\mathcal O(n^2m^2)\) \\
Diag & \(\mathcal O(nm)\) & \(\mathcal O(nm)\) & \(\mathcal O(nm)\) \\
Kron, \(n\times m\) & \(\mathcal O(n^2+m^2)\) & \(\mathcal O(n^2m+nm^2)\) & \(\mathcal O(n^3+m^3+n^2m+nm^2)\) \\
\bottomrule
\end{tabular}
\caption{\SoftServe{} method costs for updating the parameters of a $n \times m$ linear layer.}
\label{tab:costs}
\end{table}
In particular, the cost for applying and refreshing the Kronecker-factored update in \SoftServe{}-Kron is still cubic in the dimensions of the layer. To amortize this cost, we will update our estimate of the inverse Hessian only once every $K \in \mathbb{N}$ iterations. That is, once every $K$ iterations we take a snapshot of the current iterate, gradient, and sample in line~\ref{ln:snap1} in~\Cref{alg:softserve-kron}. We then use this snap shot to form our difference of iterates $\mat{S}$ and gradients $\mat{Y}$ on lines~\ref{ln:St} and~\ref{ln:Yt}. The analogous snapshot for and its application for \SoftServe{}-Diag are on lines~\ref{ln:snap1-d},~\ref{ln:st} and~\ref{ln:yt} in~\Cref{alg:softserve-diag}.

The differences of iterates and gradients are then used
to update the factorized and diagonal Hessian estimates on
line~\ref{ln:Hfactor-update} of \Cref{alg:softserve-kron}
and line~\ref{ln:Hdiag-update} of \Cref{alg:softserve-diag},
respectively.
To be more efficient (or \emph{pythonic}) we have written
the update of the diagonal estimate $h_{t+1}$ on
line~\ref{ln:Hdiag-update} in \Cref{alg:softserve-diag}
as a vector operation as opposed to an element-wise operation
as introduced in \Cref{eq:diag-update1}.

\begin{algorithm}[t]
\caption{\SoftServe{}-Kron: detailed implementation}
\label{alg:softserve-kron}
\begin{algorithmic}[1]
\Require Step size $\gamma_t$; Number of iterations $T \in \mathbb{N},$ Refresh rate $K \in [T]$,
Penalty $\lambda$; momentum parameters $\beta\in (0,1)$, Nesterov $\in \{True, False \}$
\State {\bf Initialize}  $\mat{X}_0 \in \R^{n\times m}$, $\mat{M}_{-1} =0, \,\mat{L}_0 =\mat{I}_m$ and $\mat{R}_0 = \mat{I}_n.$
\For{\(t=0,\ldots,T-1\)}
  \State Draw \(\xi_t\) and evaluate \(\mat{G}_t\gets\nabla f_{\xi_t}(\mat{X}_t)\)
  \If{\(t\bmod K=0\)} 
    \State \(\bar \xi\gets\xi_t, \; \bar{\mat{X}} \gets \mat{X}_t, \;  \bar{\mat{G}} \gets \mat{G}_t \) \label{ln:snap1} \Comment{Take snapshot}
  \EndIf
  \State  \( \mat{M}_t=\beta \mat{M}_{t-1}+(1-\beta)\mat{G}_t \) \Comment{Update Momentum buffer}  \label{ln:mom1}
  \State \(
 \hat{\mat{M}}_t=\begin{cases}
 (1-\beta)\mat{G}_t+\beta \mat{M}_t & \mbox{if Nesterov = \texttt{True}}\\
 \mat{M}_t/(1-\beta^{t+1}) & \mbox{if Nesterov = \texttt{False}}
 \end{cases}\) \label{ln:mom2}  \Comment{Use bias correction or Nesterov}  
\State \(\mat{X}_{t+1} = \mat{X}_t -\gamma_t\frac{\mat{R}_t\hat{\mat{M}}_t\mat{L}_t}{\sqrt{\tr\left( \hat{\mat{M}}_t^\top\mat{R}_t\hat{\mat{M}}_t\mat{L}_t\right) }}\)
\Comment{Update parameters} \label{ln:update-kron}
  \If{\((t+1)\bmod K=0\) and \(t+1<T\)} \Comment{Update Hessian estimate}
    \State \(\mat{S} \gets \mat{X}_{t+1}-\bar{\mat{X}}\) \label{ln:St}
    \State \(\mat{Y} \gets \nabla f_{\bar \xi}(\mat{X}_{t+1})-\bar{\mat{G}}\) \label{ln:Yt}
\State Update Kronecker factors by Gram Newton-Schulz (\Cref{alg:kron-gram}) to solve
\begin{align*}
    &\text{Update~$\mat{L}_{t+1}$~by~solving~}
    \mat{L} \, \mat{U}_L \, \mat{L} + \mat{L} = \mat{V}_L
    \\
    &\text{Update~$\mat{R}_{t+1}$~by~solving~}
    \mat{R} \, \mat{U}_R \, \mat{R} + \mat{R} = \mat{V}_R,
\end{align*}
$\qquad \quad $where $\mat{U}_L, \mat{V}_L$ are given in~\Cref{eq:ULVL}, and  $\mat{U}_R, \mat{V}_R$ are given in~\Cref{eq:URVR}. \label{ln:Hfactor-update}
    \State Rebalance the Kronecker factors 
    \begin{equation}
(\mat{L}_{t+1},\mat{R}_{t+1})=(\mat{L}_{t+1}/c,\,c\mat{R}_{t+1}) \qquad \mbox{where }\; c=\sqrt{\frac{\tr(\mat{L}_{t+1})/m}{\tr(\mat{R}_{t+1})/n}}.
\label{eq:gauge}
\end{equation}
\EndIf
\EndFor
\State \Return $\mat{X}_{T}$
\end{algorithmic}
\end{algorithm}

\begin{algorithm}[t]
\caption{\SoftServe{}-Diag detailed implementation}
\label{alg:softserve-diag}
\begin{algorithmic}[1]
\Require Step size $\gamma_t$; Number of iterations $T \in \mathbb{N},$ Refresh rate $K \in [T]$,
Penalty $\lambda$; momentum parameters $\beta\in (0,1)$, Nesterov $\in \{True, False \}$
\State {\bf Initialize}  $x_0 \in \R^d$, \(m_{-1}\gets0\) and $h_0 =(1,\ldots, 1) \in \R^d.$
\For{\(t=0,\ldots,T-1\)}
  \State Draw \(\xi_t\) and evaluate \(g_t\gets\nabla f_{\xi_t}(x_t)\)
  \If{\(t\bmod K=0\)} 
      \State \(\bar \xi\gets\xi_t, \; \bar{x} \gets x_t, \;  \bar{g} \gets g_t \) \label{ln:snap1-d} \Comment{Take snapshot}
  \EndIf
  \State  \( m_t=\beta m_{t-1}+(1-\beta)g_t \)  \label{ln:mom1-d}
  \State \(
 \hat{m}_t=\begin{cases}
 (1-\beta)g_t+\beta m_t & \mbox{if Nesterov = \texttt{True}}\\
 m_t/(1-\beta^{t+1}) & \mbox{if Nesterov = \texttt{False}}
 \end{cases}\) \label{ln:mom2-d}
  \State   $ \displaystyle x_{t+1} = x_t - \gamma_t \frac{h_t \odot \hat{m}_t}{\sqrt{\hat{m}_t^\top (h_t \odot \hat{m}_t)}}$ \Comment{Update parameters }
  \label{ln:update-d}
  \If{\((t+1)\bmod K=0\) and \(t+1<T\)} \Comment{Update Hessian estimate}
    \State \(s\gets x_{t+1}-\bar x\) \label{ln:st}
    \State \(y\gets \nabla f_{\bar{\xi}}(x_{t+1})-\bar g\) \label{ln:yt}
    \State  
 $
h_{t+1}
=
2\left(h_t+\lambda\,s\odot s\right)
\oslash
\left[
\mathbf{1}
+
\sqrt{
\mathbf{1}
+
4\lambda
\left(y\odot y\right)
\odot
\left(h_t+\lambda\,s\odot s\right)
}
\right]
$
\label{ln:Hdiag-update} 
  \EndIf
\EndFor
\State \Return \(x_T\)
\end{algorithmic}
\end{algorithm}

Our final modification to make \SoftServe{} practical for deep learning is to normalize the update. We will explain this modification for a vector parameter $x\in\R^d$. The same modification for matrix parameters follows by vectorizing, and then re-stacking as a matrix. Traditionally quasi-Newton methods are applied according to
\begin{equation}
 x_{t+1} = x_t -\gamma_t \mat{H}_t \hat{m}_t .
 \label{eq:qn-step}
\end{equation}
This update~\Cref{eq:qn-step} corresponds to the \emph{regularized} descent method:
\begin{align*}
     x_{t+1}  =\underset{x \in \R^d}{\arg\min} \;\;\hat{m}_t^\top (x-x_t) + \tfrac{1}{2\gamma_t} \| x-x_t\|_{\mat{H}_t^{-1}}^2.
\end{align*}
But now there is a growing body of evidence that the \emph{constrained} descent method~\citep{jordan2024muon,scion,crawshaw2025muonvariants} given by
 \begin{align}\nonumber
   x_{t+1}  &=\underset{x \in \R^d}{\arg\min} \;\;\hat{m}_t^\top (x-x_t) \quad \mbox{subject to } \| x-x_t\|_{\mat{H}_t^{-1}} \leq \gamma_t\\
   &= x_t - \gamma_t  \mat{H}_t \hat{m}_t \big/ \sqrt{\hat{m}_t^\top \mat{H}_t \hat{m}_t}   \label{eq:gd-const}
 \end{align}
 generally works better for deep learning. Our initial experiments also confirmed that the constrained descent~\Cref{eq:gd-const} resulting in better methods
 and thus we used \Cref{eq:gd-const} throughout our experiments
to update the parameters; see
line~\ref{ln:update-kron} of \Cref{alg:softserve-kron}
and line~\ref{ln:update-d} of \Cref{alg:softserve-diag}.

 In these updates,  note that the product $\mat{H}_t \hat{m}_t$ in~\Cref{eq:gd-const} is  given by
\begin{equation}
 \mat{H}_t \hat{m}_t =\begin{cases}
  \vecop\!\big(\mat{R}_t \hat{\mat{M}}_t\mat{L}_t \big)& \mbox{in line~\ref{ln:update-kron} in~\Cref{alg:softserve-kron}  },\\
 h_t \odot \hat{m}_t & \mbox{in line~\ref{ln:update-d} in~\Cref{alg:softserve-diag}  },
 \end{cases}
 \label{eq:block-apply}
\end{equation}
where the Kronecker case follows from \Cref{eq:kron-vec}, and \(\odot\) is the elementwise (Hadamard) product.

\subsection{Routing and state}
The parameter partition is disjoint. A Diag block stores one positive value per coordinate; a Kron block for an \(n\times m\) matrix stores \(\mat{L}\in\mathbb S_{++}^m\) and \(\mat{R}\in\mathbb S_{++}^n\). Blocks without a matrix interpretation are routed to Diag, or to a first-order fallback in the neural experiments (Appendix~\ref{app:experiments}). Kron factors of equal shape are stacked for batched matrix multiplication. The initial states are \(h_0=\mathbf1\), the all-ones vector, and \((\mat{L}_0,\mat{R}_0)=(\mat{I}_m,\mat{I}_n)\). The metric is fixed within each interval of \(K\) parameter steps. Each refresh uses one raw interval pair; we do not average pairs or interpolate metric states.

\subsection{Replay and gradient accounting}
A replay key stores every random choice that affects \(f_{\zeta_t}\): the minibatch indices, augmentation, dropout masks, and any other sampled loss input. Replaying the key at both endpoints removes sampling mismatch from \(y_t\). The replayed gradient is used only for the refresh; it does not update momentum, the learning-rate schedule, or the parameters. For \(N\) parameter steps, Algorithm~\ref{alg:softserve-kron} evaluates \(N\) ordinary gradients and \(\lfloor(N-1)/K\rfloor\) replayed ones, and we count both. With decoupled weight decay, the displacement \(s_t\) is the actual parameter displacement, including the decay. A checkpoint taken inside an interval must retain the momentum, the metric state, the interval start, the cached raw gradient, and the replay key.

\section{Experimental Details and Additional Results}
\label{app:experiments}

\subsection{Experimental setting}
\label{app:protocol}

\paragraph{Hardware and arithmetic.}
We use NVIDIA RTX PRO 6000 Blackwell GPUs for training. Models use FP32, with BF16 autocasting for GPT. 

\paragraph{Selection and evaluation.}
We tune learning rates on seed 0 using logarithmic grids, extending a grid when its minimum lies at the boundary. Final runs use seeds 1--3. Within each task, methods share initialization seeds and data. Test errors are not used for selection. Curves show three-seed means and variability bands; settings with failed seeds are shown as individual traces.

RNN Adding, MNIST, and the small PINNs screen the first-moment coefficient $\beta_1$ for \SoftServe{}, \SoftServe{}-I, Adam, Muon, and SOAP. K-FAC and K-BFGS(L) also receive a momentum search. In these studies, Nesterov on/off is tested for \SoftServe{} and Muon; Adam and SOAP retain the momentum update rules of their implementations while their $\beta_1$ values are tuned. The task descriptions below specify which settings are tuned or fixed.

\paragraph{Optimizer settings.}
\SoftServe{} uses $K=10$, and one alternating factor sweep per refresh. Kron uses the Gram QME implementation in Algorithm~\ref{alg:kron-gram}, with 18 coupled root iterations and 10 inverse iterations; Diag uses the closed-form elementwise update. 

\paragraph{Gradient accounting and plots.}
Gradient budgets count training gradients, secant replays, line-search evaluations, and loss-balancing gradients. In stochastic training, replay adds approximately 10\% to \SoftServe{}-Kron's count at $K=10$, so it takes fewer parameter updates within the same budget. 

\FloatBarrier
\subsection{RNN Adding}
\label{app:rnn}

\paragraph{Task and data.}
Each length-$T$ input has $T=100$ pairs $z_\tau=(a_\tau,m_\tau)$, where $a_\tau\sim\mathrm{Uniform}[0,1]$. Exactly two markers $m_\tau$ equal one, one in each half of the sequence. The target is the sum of the marked numbers. Every update draws 128 new sequences; validation and test sets each contain 8,192 fixed sequences per seed.

\paragraph{Model and objective.}
The network uses $h_\tau=\tanh(W_hh_{\tau-1}+W_zz_\tau+b)$, hidden width 16, $h_0=0$, and a scalar affine readout, for 321 parameters. The target and batch loss are
\begin{equation}
y=\sum_{\tau=1}^{T}m_\tau a_\tau,\qquad
\mathcal L(\theta)=\frac1B\sum_{i=1}^{B}\bigl(\widehat y_\theta(z^{(i)})-y^{(i)}\bigr)^2.
\end{equation}
Here $B$ is the batch size and $\widehat y_\theta$ is the prediction.

\paragraph{Training and tuning.}
The budget is 40,000 gradient evaluations. An 8,000-gradient pilot screens hyperparameters, followed by full-budget learning-rate sweeps selected by validation MSE (\Cref{fig:lr-rnn}). The selected values are $\lambda=999$ for \SoftServe{}-Kron and $\lambda=1$ for \SoftServe{}-Diag.

\paragraph{Evaluation.}
We report minibatch training MSE and fixed-set validation and test MSE. \Cref{fig:rnn-axes} shows training and validation MSE against parameter updates and wall-clock time. K-BFGS(L)'s two failed seeds are shown individually.

\paragraph{Local conditioning.}
We compute the full $321\times321$ loss Hessian at initialization on a fixed batch of 128 sequences for each of seeds 0--3, using FP64 differentiation. All four Hessians are indefinite. Each has one to six singular values below $321\epsilon_{64}\sigma_{\max}$, where $\epsilon_{64}=2^{-52}$ and $\sigma_{\max}$ is the largest singular value. Thus, these sampled losses are numerically nearly singular at initialization; a finite condition number cannot be resolved reliably (\Cref{fig:rnn-hessian-spectrum}).

\begin{figure}[htbp]
  \centering
  \includegraphics[width=0.5\linewidth]{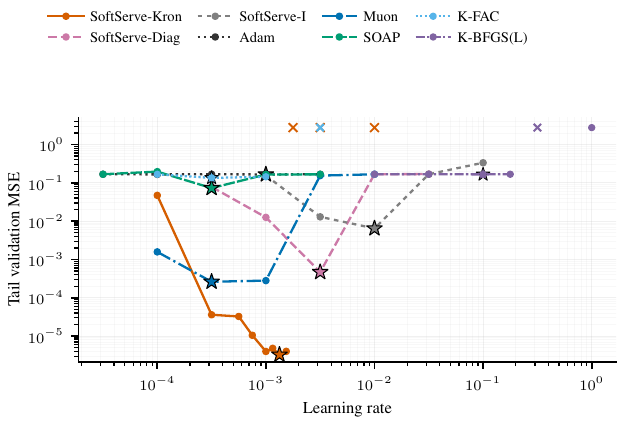}
  \caption{RNN Adding learning-rate sweeps on seed 0. Stars mark selected rates; crosses mark numerical failures.}
  \label{fig:lr-rnn}
\end{figure}
\begin{figure}[htbp]
  \centering
  \includegraphics[width=\linewidth]{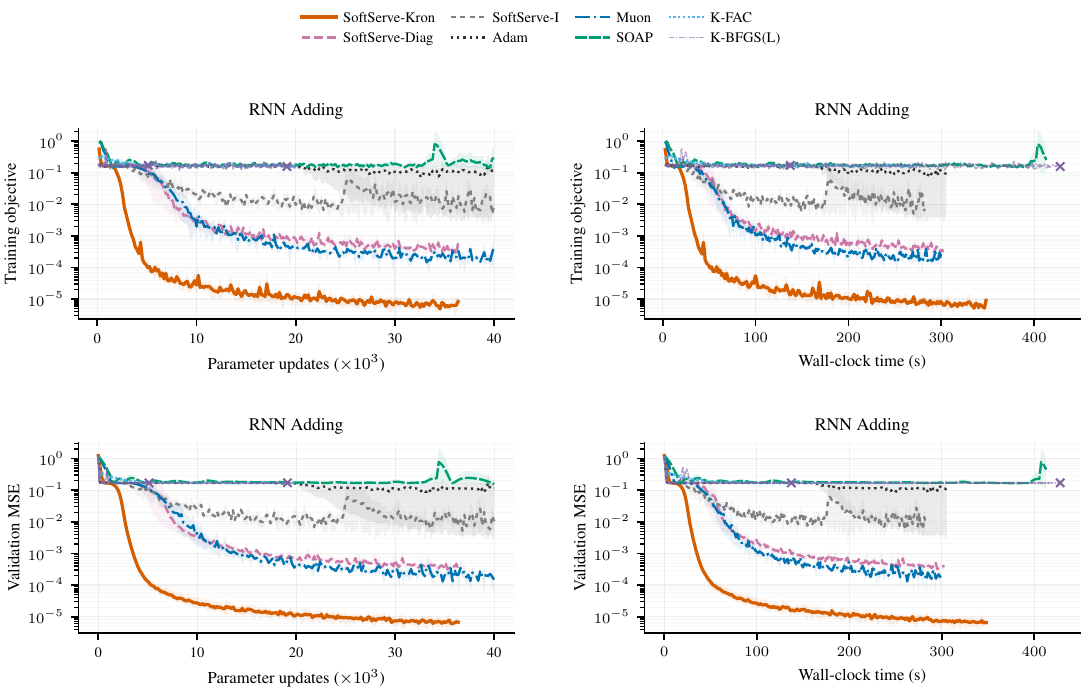}
  \caption{RNN Adding training minibatch MSE and held-out validation MSE against parameter updates and wall-clock time. Failed cohorts retain their individual traces.}
  \label{fig:rnn-axes}
\end{figure}

\begin{figure}[htbp]
  \centering
  \includegraphics[width=0.5\linewidth]{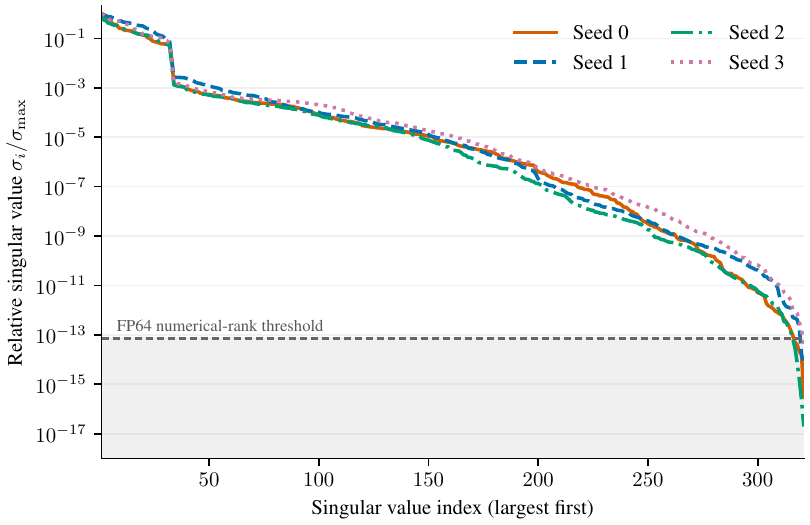}
  \caption{RNN Adding loss-Hessian singular values at initialization, normalized by the largest value. The dashed line marks the numerical-rank threshold $321\epsilon_{64}$. All four seeds have singular values below this threshold.}
  \label{fig:rnn-hessian-spectrum}
\end{figure}

\FloatBarrier
\subsection{MNIST autoencoder}
\label{app:autoencoder}

\paragraph{Task and data.}
We use MNIST, comprising $28\times28$ grayscale images of handwritten digits, with 60,000 training and 10,000 test examples and pixel values in $[0,1]$. The task is to compress each image into a low-dimensional representation and reconstruct the original input, without using digit labels. This autoencoder task is a standard benchmark for second-order optimization \citep{goldfarb2020practical}.

\paragraph{Model and objective.}
We use a 2.8M-parameter fully connected autoencoder with a symmetric encoder--decoder architecture and a 30-dimensional linear bottleneck. The remaining hidden layers use ReLU, and the output uses sigmoid. Following \citet{goldfarb2020practical}, the batch objective is
\begin{equation}
\mathcal L(\theta)=-\frac1B\sum_{i=1}^{B}\sum_{j=1}^{784}
\left[x_{ij}\log\widehat x_{ij}+(1-x_{ij})\log(1-\widehat x_{ij})\right]
+\frac{10^{-5}}2\|\theta\|_2^2,
\end{equation}
where $B$ is the batch size, $x$ is the input image, and $\widehat x$ is its reconstruction.

\paragraph{Training and tuning.}
Full-batch training uses 4,000 gradient evaluations; minibatch training uses batches of 1,000 images and 6,000 evaluations. \SoftServe{}-Kron, \SoftServe{}-I, Muon, and SOAP use Adam on biases, while \SoftServe{}-Diag acts on all parameters. Full-batch L-BFGS uses a strong-Wolfe line search.

A quarter-budget pilot screens hyperparameters, followed by full-budget learning-rate sweeps selected by the mean of the final five training objectives (\Cref{fig:autoencoder-lr}). We select $\lambda=999$ for \SoftServe{}-Kron and $\lambda=9$ for \SoftServe{}-Diag. 

\paragraph{Evaluation.}
\Cref{fig:autoencoder-test} reports test reconstruction BCE, excluding the training regularizer. \Cref{fig:autoencoder-time} shows the regularized training objective
against parameter updates and wall-clock time. Failed runs are shown individually rather than included in three-seed means.

\begin{figure}[htbp]
  \centering
  \includegraphics[width=\linewidth]{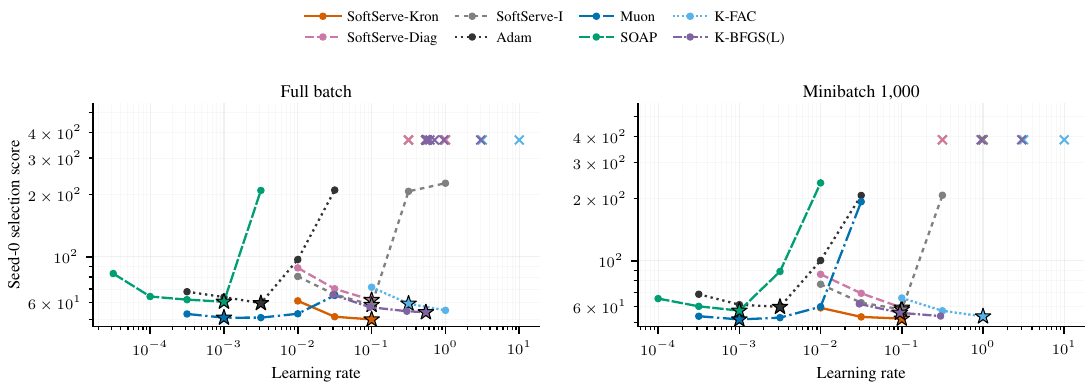}
  \caption{Autoencoder full-budget LR sweeps at selected momentum and fixed $\lambda$. Stars mark selected rates. Crosses mark failed or off-scale runs.}
  \label{fig:autoencoder-lr}
\end{figure}
\begin{figure}[htbp]
  \centering
  \includegraphics[width=\linewidth]{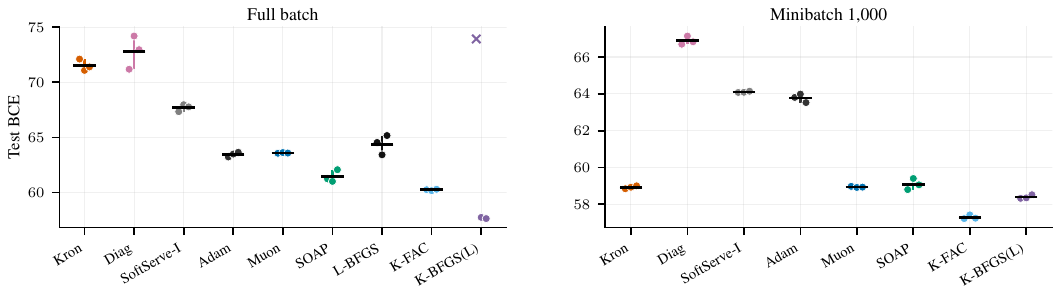}
  \caption{Autoencoder test reconstruction BCE. Dots are individual seeds; ticks are three-seed means. Full-batch K-BFGS(L) has one failed seed and therefore no three-seed mean.}
  \label{fig:autoencoder-test}
\end{figure}
\begin{figure}[htbp]
  \centering
  \includegraphics[width=\linewidth]{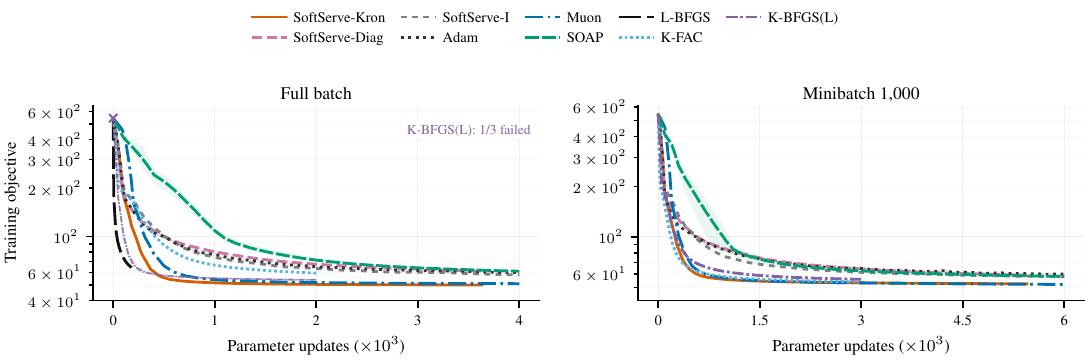}
  \includegraphics[width=\linewidth]{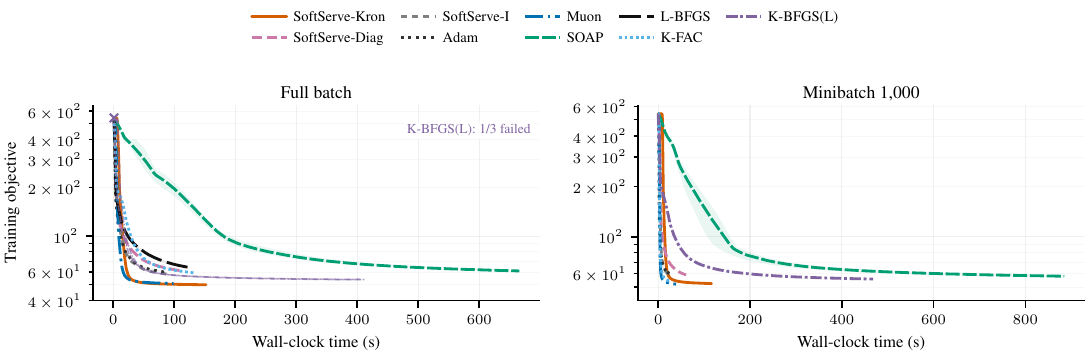}
  \caption{Autoencoder training objective against parameter updates and wall-clock time. Shading spans the three seed values.}
  \label{fig:autoencoder-time}
\end{figure}

\FloatBarrier
\subsection{Small physics-informed neural networks}
\label{app:pinn}

\paragraph{Task and data.}
We consider three one-dimensional problems from \citet{rathore2024}. The scalar function $u(x,t)$ denotes the solution at spatial coordinate $x$ and time $t$. All problems use
$t\in[0,1]$.

\medskip
\noindent\textit{Convection.}
On the spatial domain $[0,2\pi]$, we solve
\begin{align*}
\frac{\partial u}{\partial t}
+40\frac{\partial u}{\partial x}&=0,\\
u(x,0)&=\sin x,\\
u(0,t)&=u(2\pi,t).
\end{align*}
The analytical solution is
\[
u^\star(x,t)=\sin(x-40t).
\]

\medskip
\noindent\textit{Reaction.}
On the spatial domain $[0,2\pi]$, we solve
\begin{align*}
\frac{\partial u}{\partial t}-5u(1-u)&=0,\\
u(x,0)&=h(x),\\
u(0,t)&=u(2\pi,t),
\end{align*}
where $h(x)=\exp[-(x-\pi)^2/(2(\pi/4)^2)]$.
The analytical solution is
\[
u^\star(x,t)=\frac{h(x)e^{5t}}{h(x)e^{5t}+1-h(x)}.
\]

\medskip
\noindent\textit{Wave.}
On the spatial domain $[0,1]$, we solve
\begin{align*}
\frac{\partial^2 u}{\partial t^2}
-4\frac{\partial^2 u}{\partial x^2}&=0,\\
u(x,0)&=\sin(\pi x)+\tfrac12\sin(5\pi x),\\
\frac{\partial u}{\partial t}(x,0)&=0,\\
u(0,t)=u(1,t)&=0.
\end{align*}
The analytical solution is
\[
u^\star(x,t)=\sin(\pi x)\cos(2\pi t)
+\tfrac12\sin(5\pi x)\cos(10\pi t).
\]

The analytical solutions are used to evaluate solution errors, not as interior training targets.

\paragraph{Model and objective.}
We use a fully connected network with three width-200 tanh hidden layers and 81,201 parameters to predict $u_\theta(x,t)$. The training objective is
\[
\mathcal L(\theta)
=\mathcal L_{\mathrm{PDE}}
+\mathcal L_{\mathrm{IC}}
+\mathcal L_{\mathrm{BC}}.
\]
The PDE term is the mean squared equation residual. The initial- and boundary-condition terms sum the MSEs for the prescribed conditions. For Wave, these include both initial displacement and velocity, and the values at both spatial boundaries.

\paragraph{Training and tuning.}
We use 200,000 gradient evaluations and cosine learning-rate decay to one tenth of the initial rate. In the deterministic setting, the 10,000 interior collocation points remain fixed throughout training. In the stochastic setting, they are resampled at every update. Both gradients of each secant pair are evaluated on the same points.

Following hyperparameter screening, we select learning rates by the mean of the final five training losses (\Cref{fig:lr-pinn}). \SoftServe{}-Kron uses $\beta_1=0.995$. For Wave, Convection, and Reaction, respectively, its selected $\lambda$ values are $(99,9999,999)$ in the deterministic setting and 99 for all three in the stochastic setting. In the same problem order, \SoftServe{}-Diag uses $(9,1,1)$ for deterministic training and $(9,1,9)$ for stochastic training.

\SoftServe{}-Kron, \SoftServe{}-I, Muon, and SOAP update the two $200\times200$ hidden weight matrices, with Adam on the remaining parameters. Deterministic experiments also include L-BFGS and Adam followed by L-BFGS, using strong-Wolfe line searches.

\paragraph{Evaluation.}
We measure relative solution error,
\[
\frac{\|u_\theta-u^\star\|_2}{\|u^\star\|_2},
\]
with both solutions evaluated on a $129\times101$ space--time grid. \Cref{fig:pinn-l2-appendix} reports FP32 solution errors. Figures~\ref{fig:pinn-steps} and~\ref{fig:pinn-time-appendix} show training losses against parameter updates and wall-clock time.

Since numerical precision can affect PINN convergence \citep{xu2025fp64}, we repeat the experiments in FP64 using the FP32-selected configurations without retuning. Figures~\ref{fig:pinn-fp64-loss} and~\ref{fig:pinn-fp64-appendix} report the corresponding training losses and solution errors. Failed runs are shown individually. When an L-BFGS line search returns zero displacement, we extend its final value horizontally to indicate that the iterate  remains unchanged.

\begin{figure}[htbp]
  \centering
  \includegraphics[width=\linewidth]{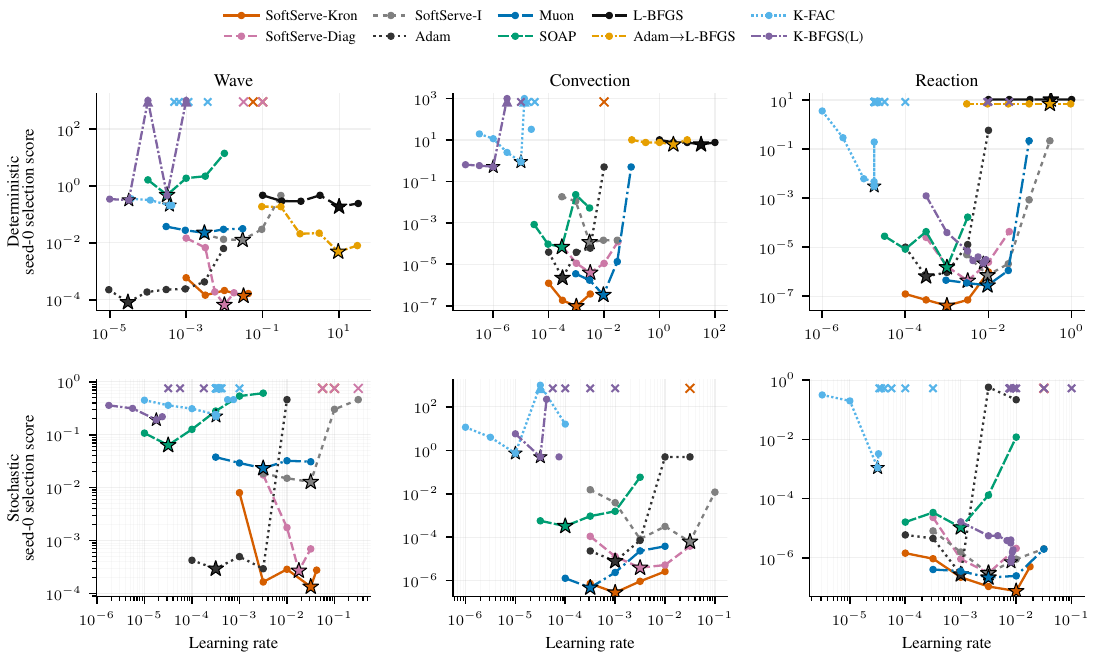}
  \caption{Small PINN learning-rate sweeps on seed 0, scored by the mean of the final five training losses. Stars mark selected rates, crosses numerical failures, and triangles off-scale finite scores.}
  \label{fig:lr-pinn}
\end{figure}
\begin{figure}[htbp]
  \centering
  \includegraphics[width=\linewidth]{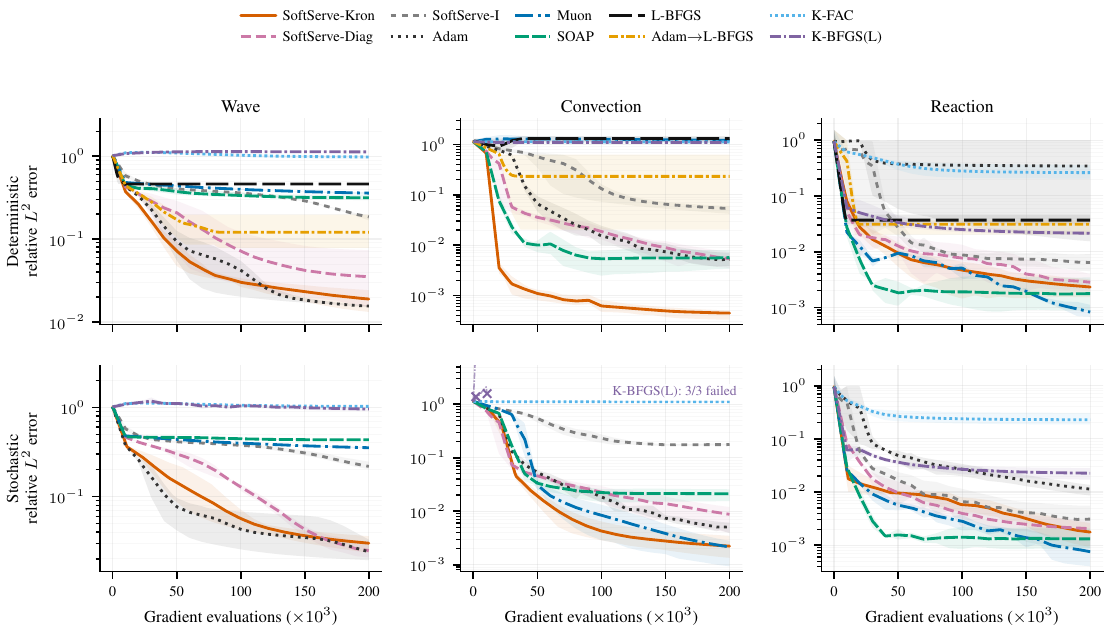}
  \caption{FP32 PINN relative solution error. Columns: Wave, Convection, Reaction. Rows: fixed and resampled collocation points. The horizontal continuation of a terminated L-BFGS run represents its unchanged solution.}
  \label{fig:pinn-l2-appendix}
\end{figure}
\begin{figure}[htbp]
  \centering
  \includegraphics[width=\linewidth]{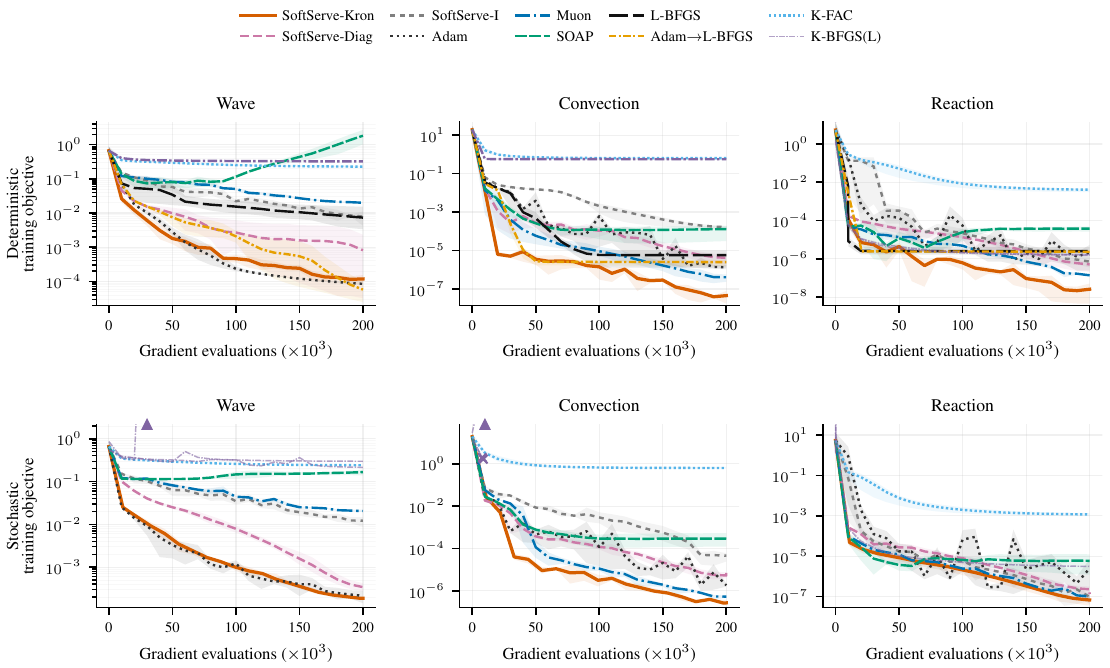}
  \caption{FP64 PINN training objectives with the FP32-selected settings. The gradient budget includes replay and line-search evaluations.}
  \label{fig:pinn-fp64-loss}
\end{figure}
\begin{figure}[htbp]
  \centering
  \includegraphics[width=\linewidth]{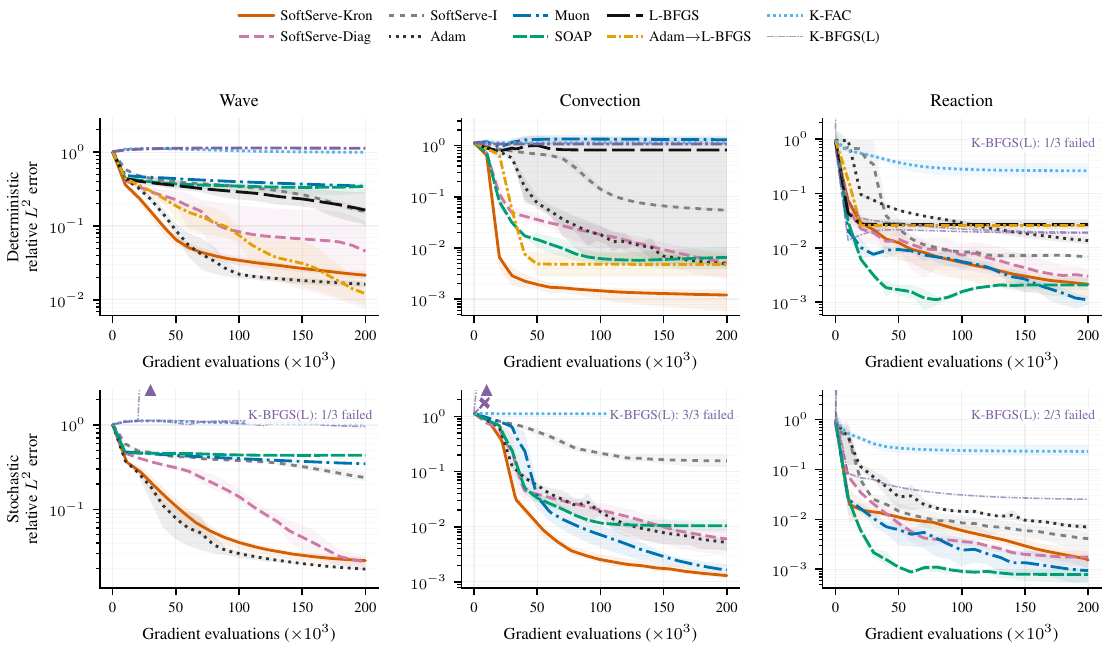}
  \caption{FP64 PINN solution errors with the FP32-selected settings. Settings with failed seeds are shown as individual traces.}
  \label{fig:pinn-fp64-appendix}
\end{figure}
\begin{figure}[htbp]
  \centering
  \includegraphics[width=\linewidth]{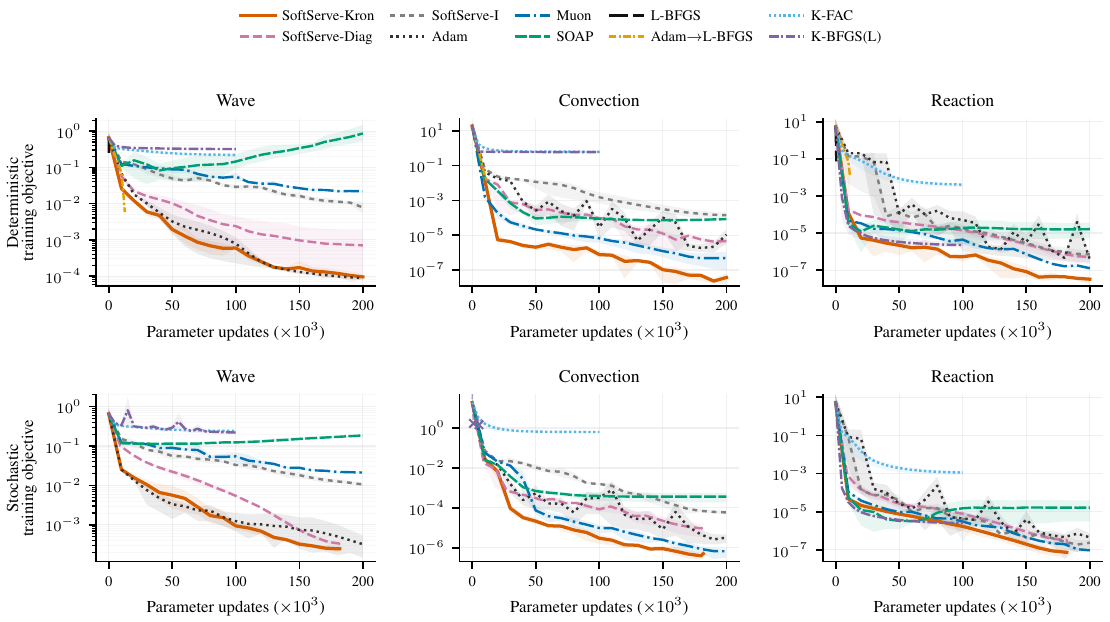}
  \caption{FP32 PINN training objectives against parameter updates.}
  \label{fig:pinn-steps}
\end{figure}
\begin{figure}[htbp]
  \centering
  \includegraphics[width=\linewidth]{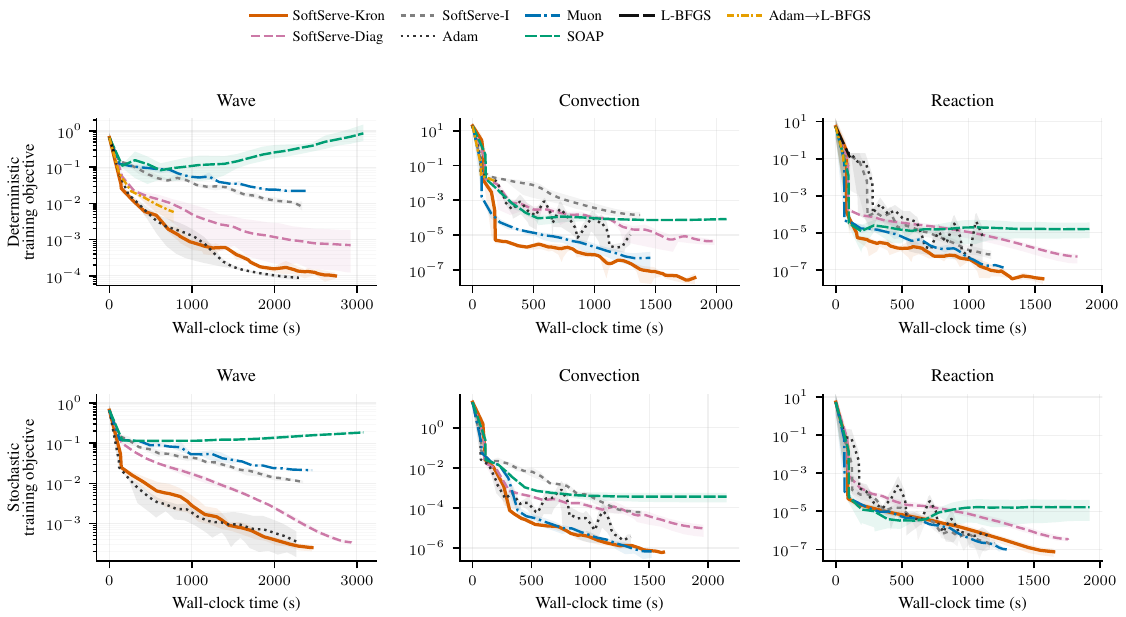}
  \caption{FP32 PINN training objectives against wall-clock time. K-FAC and K-BFGS(L) are omitted here because their timers also include construction time.}
  \label{fig:pinn-time-appendix}
\end{figure}

\FloatBarrier
\subsection{PirateNet}
\label{app:piratenet}

\paragraph{Task and data.}
We solve the Allen--Cahn and Korteweg--de Vries (KdV) equations on $(x,t)\in[-1,1]\times[0,1]$, where $u(x,t)$ is the scalar solution.

\medskip
\noindent\textit{Allen--Cahn.}
\begin{align*}
\frac{\partial u}{\partial t}
-10^{-4}\frac{\partial^2 u}{\partial x^2}
+5u^3-5u&=0,\\
u(x,0)&=x^2\cos(\pi x).
\end{align*}

\medskip
\noindent\textit{KdV.}
\begin{align*}
\frac{\partial u}{\partial t}
+u\frac{\partial u}{\partial x}
+0.022^2\frac{\partial^3 u}{\partial x^3}&=0,\\
u(x,0)&=\cos(\pi x).
\end{align*}
Both problems use periodic spatial boundaries. Reference solutions for evaluation are taken from the \texttt{jaxpi} repository.\footnote{\url{https://github.com/PredictiveIntelligenceLab/jaxpi}}

\paragraph{Model and objective.}
We use PirateNet \citep{wang2024piratenets}, a residual PINN with three width-256 blocks, tanh activations, and approximately 0.73M parameters. The network uses periodic spatial features, and its final layer is initialized by least squares against the initial condition.

Following \citet{wang2025gradient}, training minimizes a weighted combination of PDE residual and initial-condition MSEs. Causal weighting prioritizes earlier times, while gradient-norm balancing adjusts the relative weights of the loss terms. These weights are treated as constants during differentiation and held fixed between the two secant
evaluations.

\paragraph{Training and tuning.}
Each update samples 32,768 uniform collocation points. We compare methods at 300,600 gradient evaluations for Allen--Cahn and 100,200 for KdV, counting both secant replay and loss-balancing gradients. Learning rates use a 5,000-update warmup followed by exponential decay. Matrix methods use $\beta_1=0.995$ and Adam on the remaining parameters. Momentum settings are fixed during tuning.

Learning-rate sweeps run for 300,000 updates on Allen--Cahn and 100,000 on KdV, with selection based on unweighted validation loss (\Cref{fig:piratenet-lr}). The selected $\lambda$ values for \SoftServe{}-Kron are $\lambda=9999$ for Allen--Cahn and $\lambda=999$ for KdV.

\paragraph{Evaluation.}
Although all optimizers use the same weighting rules, their training weights evolve differently. We therefore compare the unweighted sum of PDE and initial-condition MSEs on common fixed points. \Cref{fig:piratenet-axes} shows this loss against parameter updates and wall-clock time.

\Cref{fig:piratenet-error} reports relative $L^2$ solution errors on the reference grid. Solution errors are not used for learning-rate selection. All curves stop at the same gradient budgets as the main figures, so \SoftServe{}-Kron takes fewer parameter updates because of replay.

\begin{figure}[htbp]
  \centering
  \includegraphics[width=\linewidth]{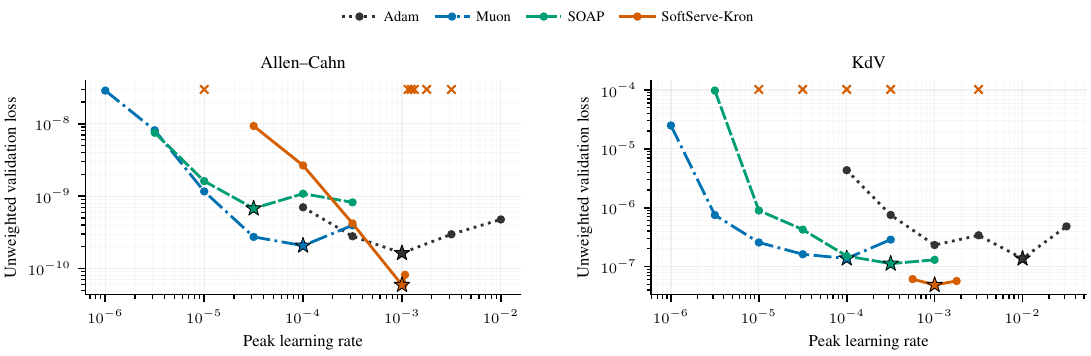}
  \caption{PirateNet full-budget learning-rate sweeps, scored by unweighted validation loss. Stars mark selected rates; crosses mark failures.}
  \label{fig:piratenet-lr}
\end{figure}
\begin{figure}[htbp]
  \centering
  \includegraphics[width=\linewidth]{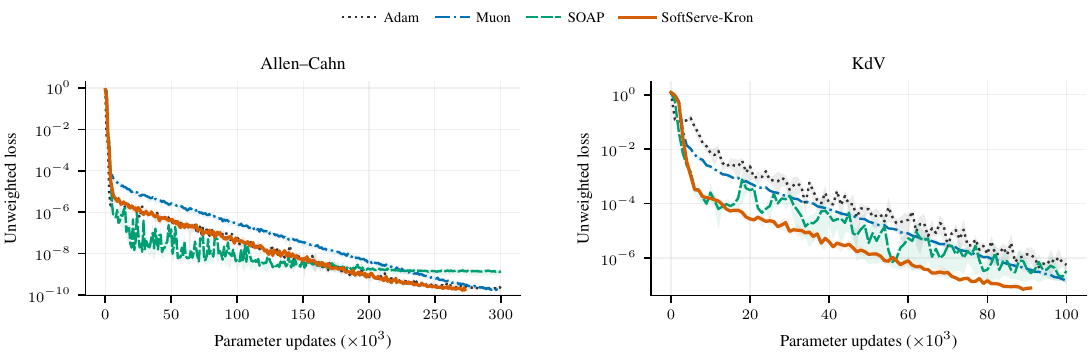}
  \includegraphics[width=\linewidth]{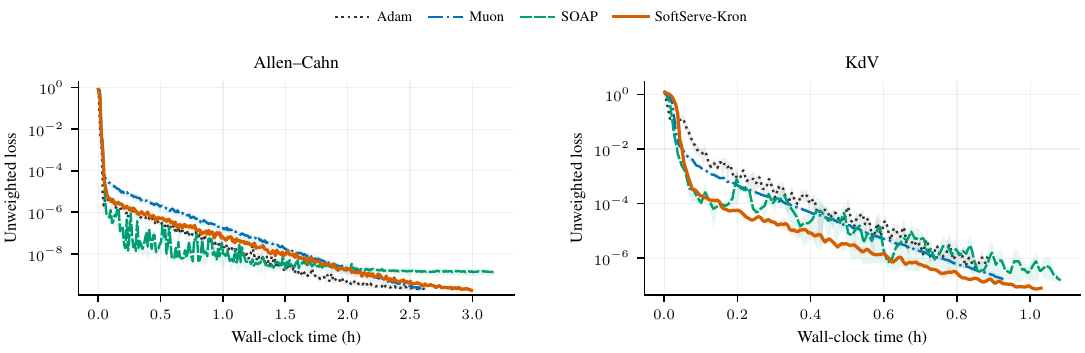}
  \caption{PirateNet unweighted PDE and initial-condition MSE against updates and wall-clock time. All curves stop at 300,600 gradient evaluations for Allen--Cahn and 100,200 for KdV; Kron therefore uses fewer parameter updates.}
  \label{fig:piratenet-axes}
\end{figure}
\begin{figure}[htb]
  \centering
  \includegraphics[width=\linewidth]{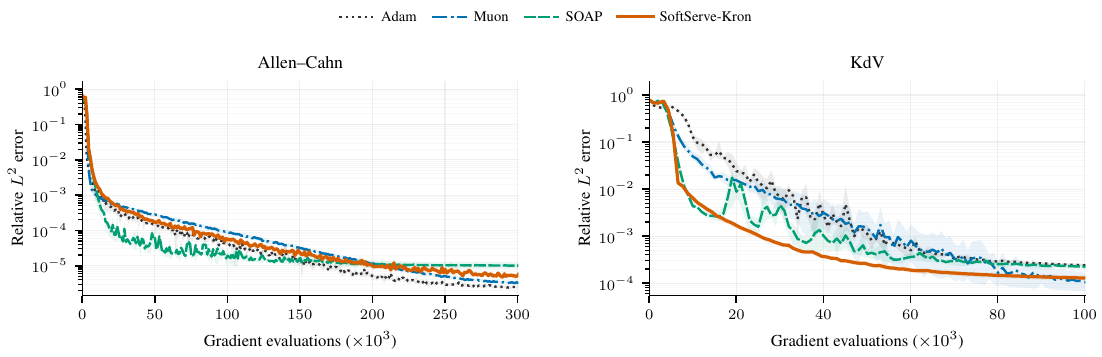}
  \caption{PirateNet relative $L^2$ error over the reference space--time grid, up to the same gradient budgets as the training-loss plots. This metric is not used to select the learning rate.}
  \label{fig:piratenet-error}
\end{figure}

\FloatBarrier
\subsection{Physics-informed diffusion}
\label{app:pidm}

\paragraph{Task and data.}
We consider structural design using the mechanics dataset of \citet{bastek2025physics}, which contains 30,000 optimized structures. Given supports, applied loads, and a material budget, the model generates a material-density field and two displacement fields. We retain the released $65\times65$ field representation and $64\times64$ finite-element grid.

\paragraph{Model and objective.}
We use the released 136M-parameter U-Net, consisting of residual convolution blocks and attention layers. It predicts the two displacement fields and the material density, with
a sigmoid output for density.

Let $z_0$ contain the target fields and $c$ the design conditions. At diffusion timestep $t$, we construct a noisy input
\[
z_t=\sqrt{\bar\alpha_t}z_0+\sqrt{1-\bar\alpha_t}\epsilon,
\qquad \epsilon\sim\mathcal N(0,I),
\]
where $\bar\alpha_t$ is the cumulative signal-retention factor. The network predicts the clean fields as $\widehat z_0=f_\theta(z_t,t,c)$. The training objective combines
reconstruction and physics-based terms:
\begin{equation}
\mathcal L(\theta)
=\mathbb E\!\left[
w_t\,\operatorname{MSE}(\widehat z_0,z_0)
+\ell_t(r_{\mathrm{eq}})
+10\,\ell_t(r_{\mathrm{vol}})
+10^{-4}C
\right].
\end{equation}
The expectation is over training examples, diffusion timesteps, and noise. The equilibrium residual is
\[
r_{\mathrm{eq}}=K(\widehat\rho)\widehat u-f,
\]
after applying boundary conditions, where $K(\widehat\rho)$ is the stiffness matrix determined by the predicted density, $\widehat u$ is the predicted displacement, and $f$ is
the applied load. The volume residual $r_{\mathrm{vol}}$ measures deviation from the material budget. Compliance $C$ measures structural flexibility under the load. We retain
the released denoising weights $w_t$ and Gaussian residual penalties $\ell_t$, including their diffusion-dependent variances, for all optimizers.

\paragraph{Training and tuning.}
We use batches of 16, 100 diffusion timesteps, constant learning rates, gradient clipping at 1, and no weight decay. An exponential moving average (EMA) of the parameters uses
decay 0.99 after the first 1,000 updates.

Only learning rates are tuned. Each method initially tests five rates for 10,000 updates, extending the grid when needed. Selection uses the EMA validation objective on 128
fixed examples with fixed noise (\Cref{fig:pidm-lr}). Final runs use a budget of 100,000 gradient evaluations.

Momentum and $\lambda$ remain fixed. Muon uses $\beta_1=0.95$ with Nesterov; Adam, SOAP, and \SoftServe{}-Kron use $\beta_1=0.9$. Adam and SOAP use $\beta_2=0.999$.
\SoftServe{}-Kron uses $\lambda=9$ without Nesterov. Matrix methods use the same convolution-kernel partitions, with blocks of at most $256\times256$, and Adam at learning rate
$10^{-4}$ on the remaining parameters.

\paragraph{Evaluation.}
We evaluate the training objective on 32 fixed examples with fixed noise using the current, unaveraged parameters. Validation uses the EMA parameters on 128 separate fixed
examples. \Cref{fig:pidm-axes} shows both objectives against gradient evaluations, parameter updates, and wall-clock time. All curves stop at 100,000 gradient evaluations.
Including secant replay, this permits approximately 90,900 parameter updates for \SoftServe{}-Kron and 100,000 for the baselines.

\begin{figure}[htbp]
  \centering
  \includegraphics[width=0.65\linewidth]{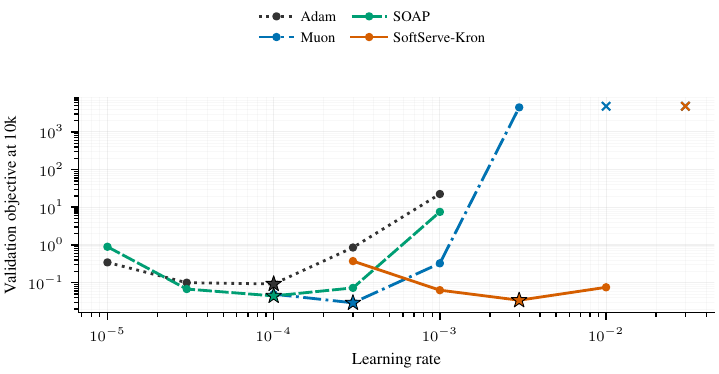}
  \caption{PIDM LR selection after 10,000 updates. The score is the EMA validation objective. Stars mark selected rates; crosses mark numerical failures.}
  \label{fig:pidm-lr}
\end{figure}
\begin{figure}[htbp]
  \centering
  \includegraphics[width=\linewidth]{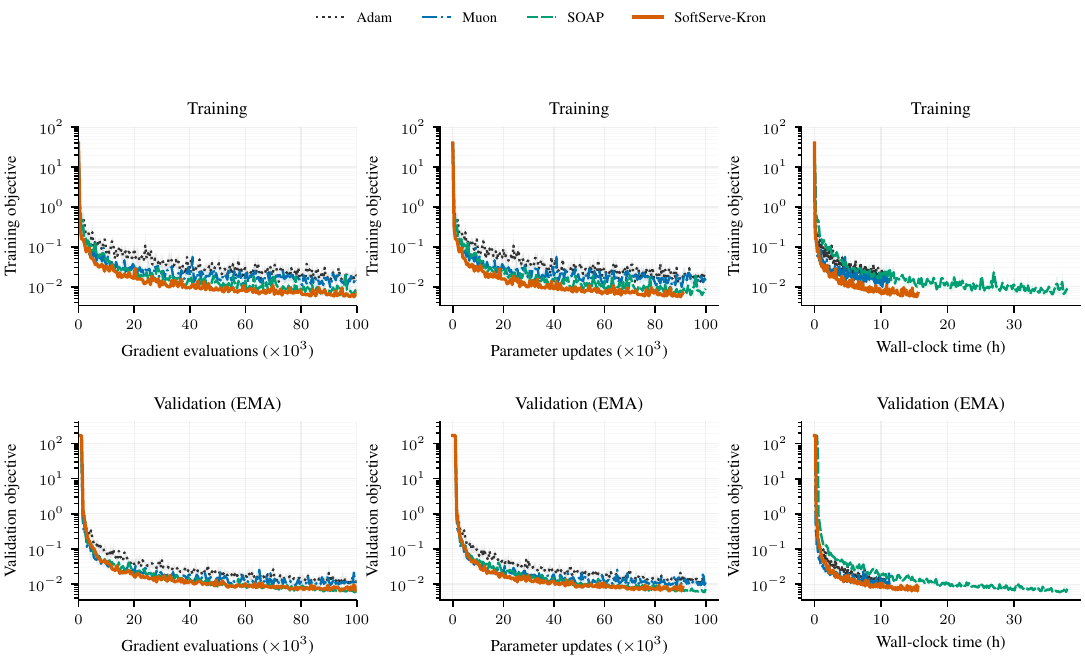}
  \caption{PIDM training objective (unaveraged parameters) and validation objective (EMA parameters) against gradient evaluations, updates, and wall-clock time. Every curve stops at 100,000 gradient evaluations, corresponding to approximately 90,900 parameter updates for Kron and 100,000 for the baselines.}
  \label{fig:pidm-axes}
\end{figure}

\FloatBarrier
\subsection{Language model}
\label{app:llm}

\paragraph{Task and data.}
We train a language model on 200M tokens from FineWeb \citep{penedo2024fineweb}. The task is to predict each token from its preceding context.

\paragraph{Model and objective.}
We use a 16.17M-parameter GPT-style model \citep{radford2019language} with four layers, four attention heads, hidden width 256, and context length 512. We minimize next-token
cross-entropy, reported as negative log-likelihood (NLL) per token.

\paragraph{Training and tuning.}
Each update uses 64 sequences, giving 32,768 tokens per full update. The 200M-token budget requires 6,104 updates, with a smaller final batch. Kron additionally uses 610 same-batch gradient replays. The learning rate warms up over the first 2\% of training, remains constant, then decays by cosine to one tenth of its peak over the final 10\%. We use no dropout, weight decay, or gradient clipping.

We tune learning rates on seed 0 using five-point sweeps at the full 200M-token budget, selecting by validation NLL. We first tune AdamW, then use its selected peak learning rate of $0.003$ for the AdamW updates outside the matrix methods. Muon, SOAP, and Kron update the same attention and MLP weight matrices, with AdamW on the remaining parameters. The selected peak learning rates are $0.003$ for Muon and SOAP, and $0.01$ for Kron.

Kron uses $\lambda=99$, $\beta_1=0.99$, $K=10$, and no Nesterov. Muon uses $\beta_1=0.9$ with Nesterov. SOAP uses betas $(0.95,0.95)$ and refreshes its preconditioner every 10 updates. AdamW uses betas $(0.9,0.95)$, both as a baseline and for the remaining parameters of the matrix methods. Momentum, Nesterov, and $\lambda$ are fixed throughout tuning.

\paragraph{Evaluation.}
We evaluate the selected configurations on fresh seeds 1--3 and report their mean NLL. Final training NLL uses one million training tokens; learning-rate selection and final validation NLL use five million held-out tokens.

\FloatBarrier
\section{Ablation studies}
\label{app:ablations}

We study the secant penalty $\lambda$, the accuracy of the Gram QME solver, and the refresh interval $K$ for \SoftServe{}-Kron. Training studies use RNN Adding and minibatch MNIST, with budgets of 40,000 and 6,000 gradient evaluations, respectively, including secant replay. For each setting, we select the learning rate on seed 0 and evaluate on fresh seeds 1--3. RNN selection uses validation MSE; MNIST selection uses the mean of the last five training objectives. Test errors are not used for selection. Matching completed runs are reused.

Unless varied, $\lambda=999$, $K=10$, and the solver uses 18 coupled root iterations and 10 inverse iterations. RNN uses $\beta_1=0.99$ and MNIST uses $\beta_1=0.95$, both with Nesterov momentum. Lines and points show arithmetic means; bands and error bars span the minimum and maximum over the three seeds.

\subsection{Sensitivity to the secant penalty}
\label{app:ablation-lambda}

We vary $\lambda$ from 0.01 to 9999 on RNN Adding, retuning the learning rate for each value. Figure~\ref{fig:ablation-lambda} shows lower mean validation error as $\lambda$ increases from 0.01 to 999. Values from 0.1 to 999 outperform the tuned Muon baseline in this experiment. Increasing $\lambda$ further to 9999 causes numerical failure on all three evaluation seeds. Thus, the penalty can substantially affect both accuracy and stability; the best value here is not a task-independent default. 

\begin{figure}[htbp]
  \centering
  \includegraphics[width=0.66\linewidth]{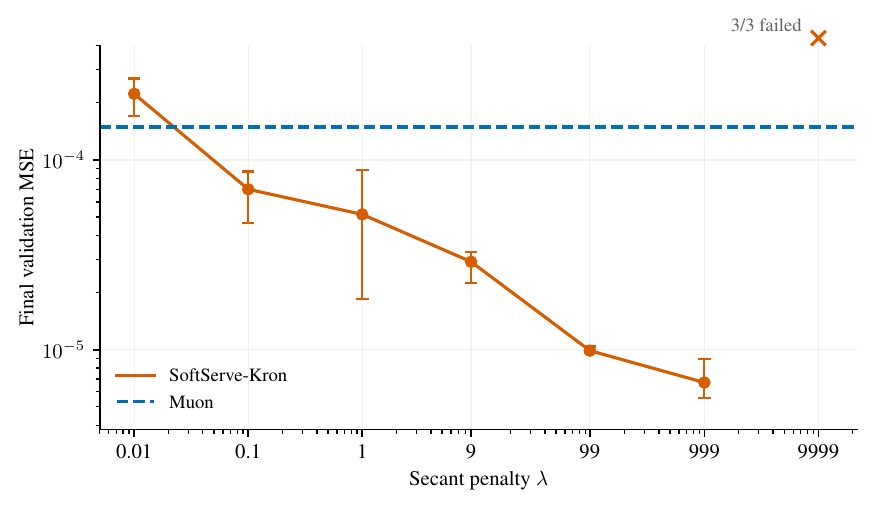}
  \caption{RNN Adding validation MSE after 40,000 gradient evaluations, with a separately selected learning rate for each $\lambda$. Points show means and bars show seed ranges. The dashed line is the mean of the completed Muon comparison. The cross above the axis denotes three numerical failures at $\lambda=9999$, not a measured loss.}
  \label{fig:ablation-lambda}
\end{figure}

\FloatBarrier
\subsection{Accuracy of the QME solver}
\label{app:ablation-qme}

We first isolate numerical accuracy by solving $\mat{Q}\mat{U}\mat{Q}+\mat{Q}=\mat{V}$ for fixed $256\times256$ matrices. We construct $\mat{U}=\mat{R}\mat{R}^{\top}$ with singular values of $\mat{R}$ between 0.7 and 1.3, and use independent random eigenbases for $\mat{U}$ and $\mat{V}$. The eigenvalues of $\mat{V}$ are logarithmically spaced from 1 to $\kappa\in\{10,10^3,10^5\}$. We vary both iteration counts in the FP32 Gram solver and compare with the FP64 decomposition-based solution $\mat{Q}_\star$ using relative Frobenius error $\|\widehat{\mat{Q}}-\mat{Q}_\star\|_{\mathrm F}/\|\mat{Q}_\star\|_{\mathrm F}$.

Figure~\ref{fig:ablation-qme-heatmap} shows that too few iterations can leave substantial error. At the default 18/10 counts, median errors are approximately $2.3\times10^{-6}$, $7.3\times10^{-5}$, and $7.2\times10^{-3}$ for the three values of $\kappa$. More iterations do not uniformly reduce FP32 error, particularly once rounding error limits accuracy.

\begin{figure}[htbp]
  \centering
  \includegraphics[width=\linewidth]{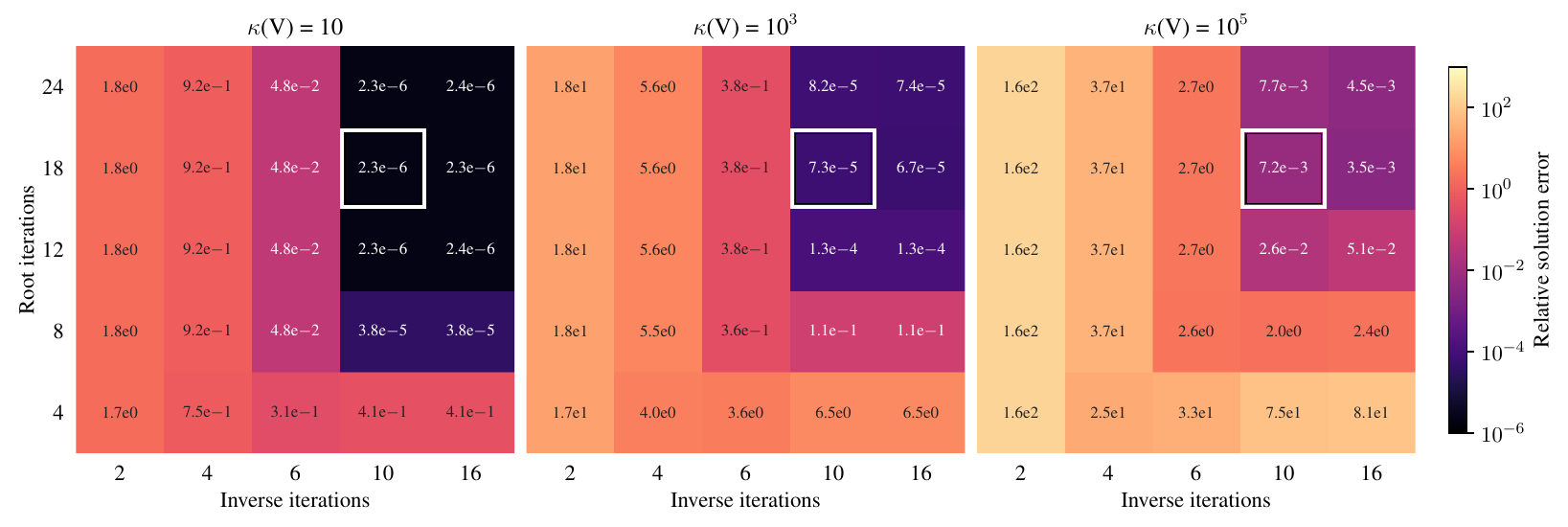}
  \caption{Relative QME solution error as the root and inverse iteration counts vary. Each cell is the median over three matrix seeds; all panels share a logarithmic color scale. Lower values are better. White outlines mark the default 18 root and 10 inverse iterations.}
  \label{fig:ablation-qme-heatmap}
\end{figure}

We then vary the solver counts during RNN training (Figure~\ref{fig:ablation-qme-rnn}). We compare 12, 18, and 24 root iterations at 10 inverse iterations, and 10, 16, and 24 inverse iterations at 18 root iterations. Separate sweeps select a learning rate of $1.33\times10^{-3}$ for all four stable settings. Their learning curves are similar; increasing either count beyond 18/10 does not consistently improve validation error. With 12/10, all three tested seed-0 learning rates fail before the full budget, so no fresh-seed mean is shown.

\begin{figure}[htbp]
  \centering
  \includegraphics[width=0.8\linewidth]{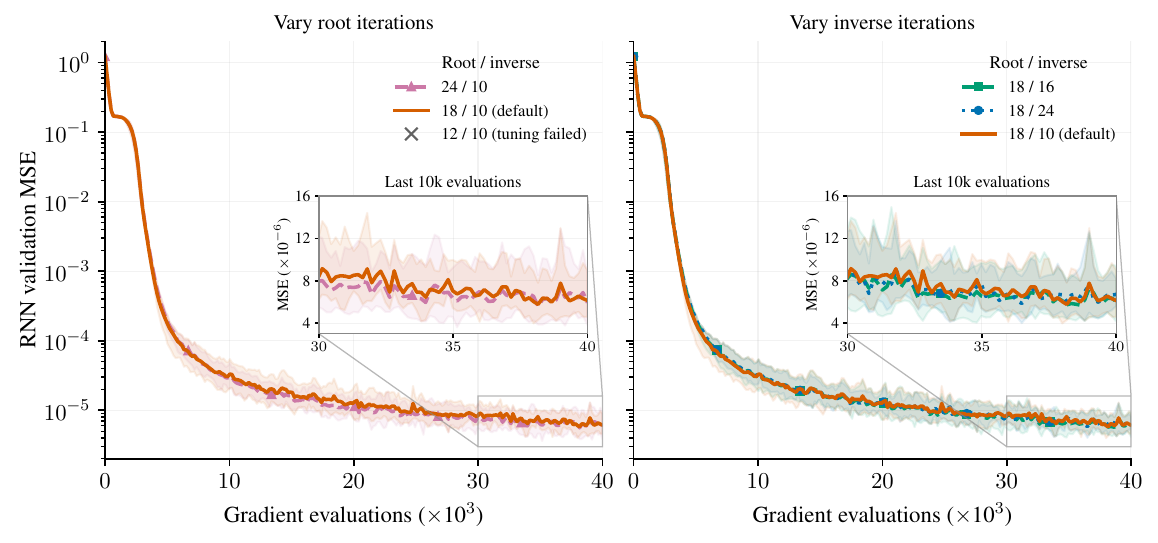}
  \caption{RNN validation MSE as root iterations vary (left) and inverse iterations vary (right). Legend entries give root/inverse counts; 18/10 is the default. Each completed setting uses a separately selected learning rate and three fresh seeds. Lines show means and shading spans the seed minimum and maximum. Insets enlarge the last 10,000 gradient evaluations on a linear MSE axis with shared limits. The 12/10 setting failed at all three tested tuning rates and has no learning curve.}
  \label{fig:ablation-qme-rnn}
\end{figure}

\FloatBarrier
\subsection{Refresh frequency and training cost}
\label{app:ablation-refresh}

We compare $K\in\{1,5,10,20,50\}$, retuning the learning rate at each interval. Figure~\ref{fig:ablation-refresh} shows learning curves against gradient evaluations and wall-clock time. More frequent refreshes require more secant replays and leave fewer parameter updates within the same budget. Refreshing at every update slows convergence on both tasks. The best trade-off is task-dependent: $K=5$ reaches the lowest mean RNN validation MSE, while larger intervals shorten MNIST training with similar final objectives.

\begin{figure}[htbp]
  \centering
  \includegraphics[width=\linewidth]{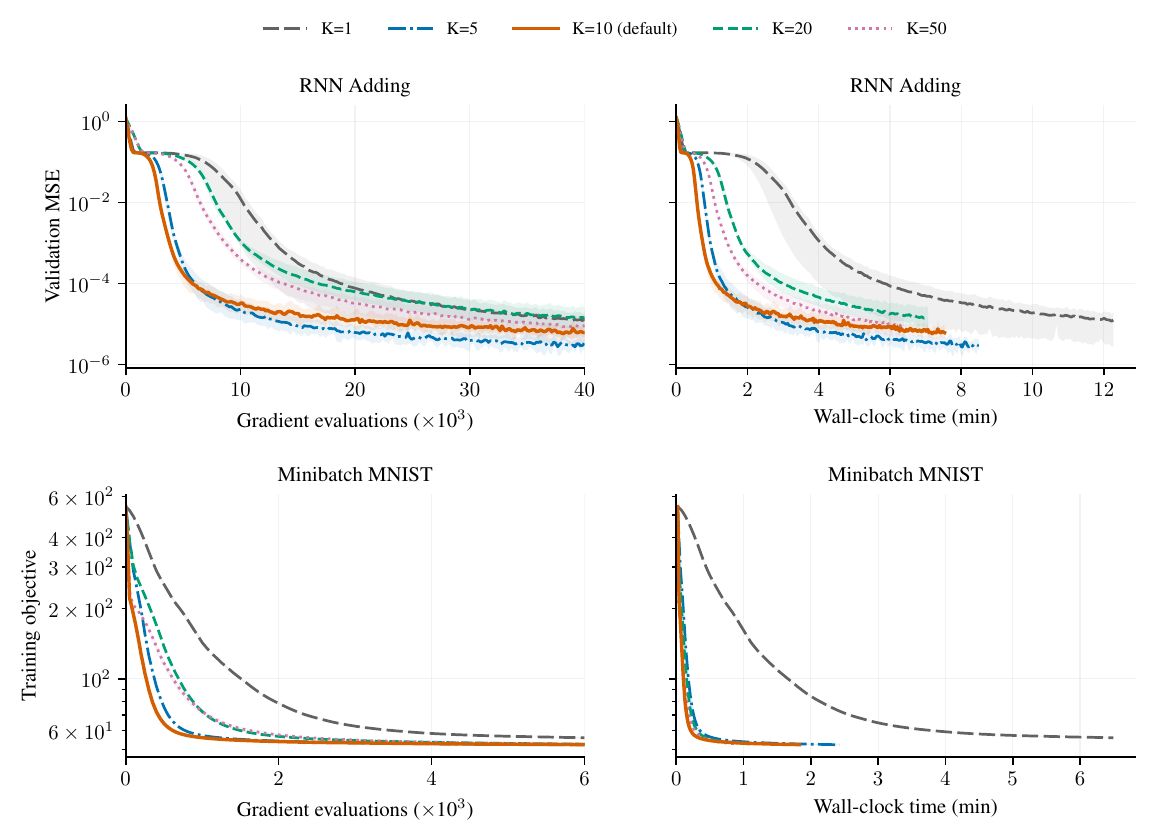}
  \caption{Learning curves for different refresh intervals. Rows show RNN validation MSE and minibatch MNIST training objective; columns show gradient evaluations and wall-clock time. All runs stop at the same task-specific gradient budget, including replay: 40,000 for RNN and 6,000 for MNIST. Lines show three-seed means and shading spans the seed minimum and maximum. Wall-clock time includes evaluation and diagnostic overhead. Orange lines mark the default $K=10$.}
  \label{fig:ablation-refresh}
\end{figure}
\end{document}